\documentclass{article}
\PassOptionsToPackage{numbers, compress}{natbib}

\usepackage[final]{neurips_2022}

\usepackage[utf8]{inputenc} 
\usepackage[T1]{fontenc}    

\usepackage{url}            
\usepackage{booktabs}       
\usepackage{amsfonts}       
\usepackage{nicefrac}       
\usepackage{microtype}      
\usepackage{xcolor}         

\usepackage{pdfpages}
\usepackage{selectp}

\usepackage{amsmath,amsfonts,bm}

\def\eqref#1{equation~\ref{#1}}

\def\1{\bm{1}}

\def\vw{{\bm{w}}}
\def\vx{{\bm{x}}}

\DeclareMathAlphabet{\mathsfit}{\encodingdefault}{\sfdefault}{m}{sl}
\SetMathAlphabet{\mathsfit}{bold}{\encodingdefault}{\sfdefault}{bx}{n}

\newcommand{\E}{\mathbb{E}}

\newcommand{\R}{\mathbb{R}}

\DeclareMathOperator*{\argmax}{arg\,max}

\newcommand{\DataX}{\mathcal{X}}

\newcommand{\DataS}{\mathcal{S}}

\newcommand{\dotp}[2]{\left<#1, #2\right>}

\newcommand{\normsm}[1]{\| #1 \|}

\newcommand{\normtwosm}[1]{\normsm{#1}_2}

\newcommand{\norminfsm}[1]{\normsm{#1}_\infty}

\newcommand{\abs}[1]{\left\lvert #1 \right\rvert}
\newcommand{\abssm}[1]{\lvert #1 \rvert}

\newcommand{\onec}[1]{\mathbbm{1}_{[#1]}}

\newcommand{\sphS}{\mathbb{S}}

\newcommand{\sgn}{\mathrm{sgn}}
\newcommand{\TV}{\mathrm{TV}}

\newcommand{\cs}{consistency score}
\usepackage{mycustom}
\usepackage{hyperref}       
\usepackage{caption}
\usepackage{subcaption}

\usepackage{bbm}

\usepackage[shortlabels]{enumitem}
\setlist[enumerate,1]{leftmargin=0.6cm}
\setlist[itemize]{leftmargin=0.6cm}

\usepackage{microtype}
\usepackage{graphicx}
\usepackage{booktabs} 

\usepackage{hyperref}

\usepackage{amsmath}
\usepackage{amssymb}
\usepackage{mathtools,thm-restate}
\usepackage{amsthm}

\usepackage[capitalize,noabbrev]{cleveref}

\theoremstyle{plain}
\newtheorem{theorem}{Theorem}[section]

\theoremstyle{definition}
\newtheorem{definition}[theorem]{Definition}

\theoremstyle{remark}

\usepackage{xcolor}
\usepackage{graphicx}
\usepackage{enumitem}

\newcommand\dy[1]{\textcolor{olive}{DY: #1}}

\newcommand\revise[1]{\textcolor{black}{#1}}

\newcommand{\kf}[1]{{\color{orange}Kaifeng: #1}}
\newcommand{\map}{\bm{m}}
\newcommand{\method}{\mathcal{M}}
\newcommand{\mask}{M}

\newcommand{\textbasemetric}{base metric}

\newcommand{\textBaseMetric}{Base Metric}
\newcommand{\AUC}{g_{\textrm{AUC}}}

\newcommand{\ex}{\mathop{\E}}
\usepackage{authblk}
\usepackage{comment}

\usepackage[textsize=tiny]{todonotes}

\title{New Definitions and Evaluations for Saliency Methods:
Staying Intrinsic, Complete and Sound}

\author[*,1]{\bf Arushi Gupta}
\author[*,1]{\bf Nikunj Saunshi}
\author[*,1]{\bf Dingli Yu}
\author[1]{\bf Kaifeng Lyu}
\author[1]{\bf Sanjeev Arora}
\affil[1]{Princeton University\protect\\\texttt{\{arushig,nsaunshi,dingliy,klyu,arora\}@cs.princeton.edu}\vspace{0.1in}}
\affil[*]{Denotes equal contribution}

\begin{document}


\maketitle

\begin{abstract}
Saliency methods compute heat maps  that highlight portions of an input that were most {\em important} for the label assigned to it by a deep net. Evaluations of saliency methods  convert this heat map into a new {\em masked input} by retaining the $k$ highest-ranked pixels of the original input and replacing the rest with \textquotedblleft uninformative\textquotedblright\ pixels, and checking if the net's output is mostly unchanged. This is usually seen as an {\em explanation} of the output, but the current paper highlights  reasons why this inference of causality may be suspect. Inspired by logic concepts of {\em completeness \& soundness}, it observes that the above type of evaluation focuses on completeness of the explanation, but ignores soundness.  New evaluation metrics are introduced to capture both notions, while staying in an {\em intrinsic} framework---i.e., using the dataset and the net, but no separately trained nets, human evaluations, etc. A simple saliency method is described that matches or outperforms prior methods in the evaluations. Experiments also suggest new intrinsic justifications, based on soundness, for popular heuristic tricks such as TV regularization and upsampling.

\end{abstract}
    \section{Introduction}
\label{sec:intro}

{\em Saliency methods}  try to understand why the deep net gave a certain answer on a particular input, and are an important component of  explainability, fairness,  robustness, etc., in deep learning. 
This paper restricts attention to 
the (large) class of saliency methods that 
return an importance score for each coordinate of the input --- often visualized as a  heat map --- which captures the coordinate's importance to the final decision.\footnote{Heat maps suffice for recognition/classification tasks; other tasks may require more complex explanations.}  Early saliency methods used an axiomatic system of ``credit attribution'' to individual input coordinates,  using backpropagation-like methods \citep{binder2016layer,selvaraju2016grad} and cooperative game theory, such as Shapley values~\citep{lundberg2017unified,yeh2020completeness}. See \Cref{sec:prior} and \citet{samek2019explainable} for discussion of strengths and limitations of such methods. 
More recent methods try to find heat maps that are largely concentrated on a smaller set of pixels/coordinates and are discussed further below.

\looseness-1 There exists an ecosystem of {\em evaluation metrics} to  evaluate saliency methods. Evaluations can be {\em extrinsic}, involving   human evaluation \citep{adebayo2018sanity}, and comparison to certain ground truth explanations \citep{zhang2018top}. We restrict attention to
{\em intrinsic} evaluations, which use computations involving the net itself and the heat/saliency map --- but no human evaluations or evaluations that involve training new deep nets. Popular intrinsic evaluations include {\em saliency metric}~\citep{dabkowski2017real} and insertion \& deletion metrics~\citep{petsiuk2018rise}.

\looseness-1 A frequent idea in intrinsic evaluations (see \Cref{sec:stat} for references) is to create a new composite input --- or sequence of such inputs --- using the heat map and the original input, and to evaluate the original net on this composite input. For example, if $\mask$ is a binary vector with $1$'s in the $k$ coordinates with the highest values in the heat map, then $x \odot \mask$ (with $\odot$ denoting coordinate-wise multiplication) can be viewed as a masked input where only $k$ of the original coordinates of $x$ remain. It is customary to replace the  zeroes in this masked input with ``uninformative'' values, which we refer to as {\em gray} pixels. 
This masked input is fed into the original net\footnote{This raises a potential issue of distribution shift because the net never trained on masked inputs. In practice, especially for image data, trained nets continue to work fine on masked inputs.} to check if the net outputs the same label as on the full input --- and, if so, one should conclude that the unmasked coordinates of $x$ were salient to the net's output. 
 Not surprisingly, recent methods~\citep{fong2017interpretable,dabkowski2017real,phang2020investigating} for finding heatmaps use an objective that tries to directly optimize performance on such evaluations, greatly outperforming older axiom-based methods.

{\em Should we trust such mask-based explanations}? At first sight, this question may appear naive. In the above-described masked input
$x \odot M$, some $k$ pixels were copied from the original image and the remaining pixels were made \textquotedblleft gray.\textquotedblright~
If the original net outputs essentially the same label as it did on the full image,  does this not {\em prove} that the portion that came from the mask $M$  must have {\em caused} the net's output on the full image?  The answer is \textquotedblleft No\textquotedblright, because the {\em positions} of the gray pixels can carry some signal.

\Cref{fig:artifact-small} shows that  in standard  datasets,  it is possible to create masks out of most inputs such that the masked input  can cause the net to assign high probability for a {\em different}  label than in its original prediction.  Prior works referred to this phenomenon as \textquotedblleft mask artifacts\textquotedblright\ and suggested using total variation (TV) regularization to mitigate this issue. But \Cref{fig:artifact-small} shows TV regularization is not a full solution to masks that cause the net to (incorrectly) flip the answer.

\myparagraph{Soundness needed.} The masked input method is trying to {\em prove} that a certain portion of the image  \textquotedblleft caused\textquotedblright\ the net's output.  Usually logical reasoning must display both {\em completeness} (i.e., all correct statements are provable) and {\em soundness} (incorrect statements cannot be proved).
Checking that the net's output is essentially unchanged with the masked input is akin to verifying {\em completeness}.

But this by itself should not be convincing because, as mentioned, it is possible to use the method to find a different mask to justify another (i.e., wrong) label.   {\em Soundness}  would require verifying that the same saliency method cannot be used to produce masked inputs that make the net output a different label.
Together, completeness and soundness ensure that the evaluation of the masked inputs produced using various labels approximately tracks the model's probability of assigning the label (see Section~\ref{sec:completeness_and_soundness}).

\begin{figure}[t!]
    \vspace{-0.1in}
    \begin{subfigure}{.5\textwidth}
        \centering
        \includegraphics[width=0.95\textwidth]{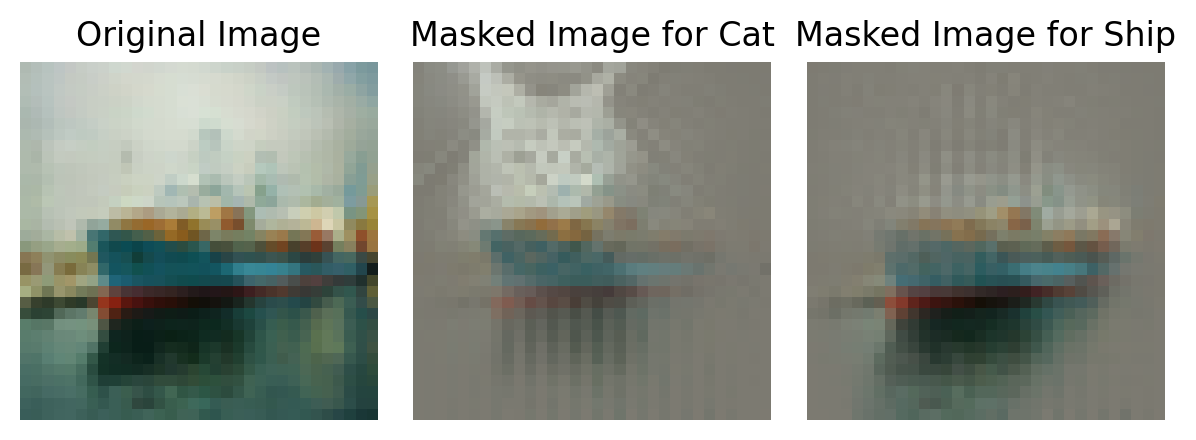}
    \end{subfigure}%
    \begin{subfigure}{.5\textwidth}
        \centering
        \includegraphics[width=0.95\textwidth]{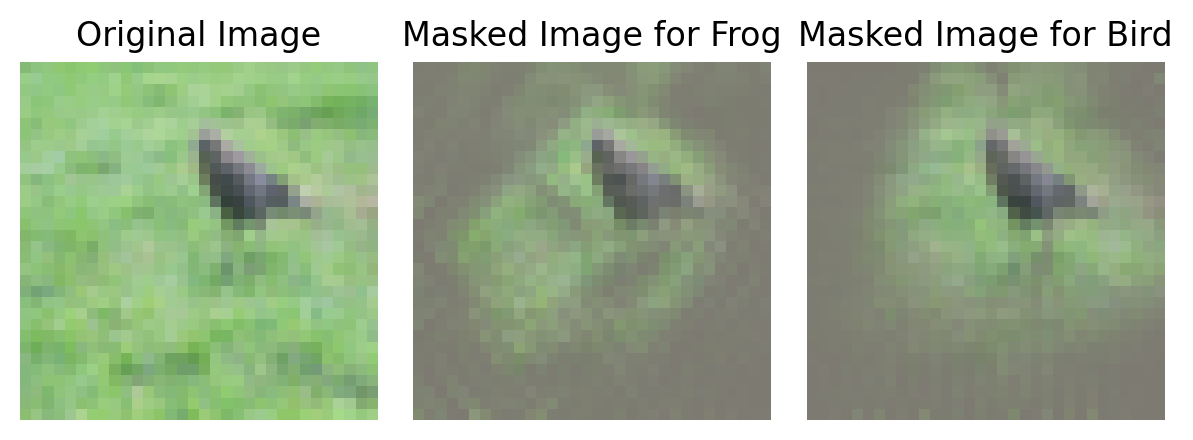}
    \end{subfigure}
    \caption{Masked CIFAR-10 images generated by our procedure with TV regularization.
    ``Artifacts'' exist for masks generated for incorrect labels; more examples can be found in \Cref{fig:artifact} in Appendix.
    The base model outputs the correct label on the original image (ship and bird resp.) with probability at least 0.99, and assigns probability at most $10^{-5}$ for the incorrect label (cat and frog resp.). With the generated masks, the AUC metric (defined in \Cref{sec:methods_metrics}) for the correct label remains high (around 0.94 and 0.90), which corresponds to {\em completeness}, but AUC metric for the incorrect label rises tremendously (around 0.18 for cat mask, 0.71 for frog mask.) This suggests violation of {\em soundness}. }
    \label{fig:artifact-small}
    \vspace{-0.2in}
\end{figure}

\noindent{\bf Main Contributions of this paper:} 
\vspace{-0.1in}
\begin{itemize}

\item \Cref{sec:completeness_and_soundness} draws connections to logical proof systems and formalizes completeness and soundness in the context of saliency methods.
These are intrinsic evaluations of saliency without appealing to human judgements (e.g., whether or not machine-generated masks contain artifacts).
Our notion of completeness and soundness requires saliency methods to output a map for {\em every label} and not just the model prediction.

\item \Cref{sec:procedures} revisits methods that learn masks through optimization.
Incorporating soundness, we propose a simple optimization method that maximizes the probability for any given label\footnote{Computation cost proportional to the number of labels is only incurred while designing/evaluating the saliency method but not deployment time. Another alternative is to verify soundness condition for top $k$ labels.}, rather than only the label deemed most likely by the net.
A key design choice of this method is the pixel replacement strategy during training: non-salient pixels are replaced with pixels of a random image, instead of using gray pixels, blurring or counterfactual models. \footnote{Code for computing completeness, soundness, and our masking method available at \url{https://github.com/agup/soundness\_saliency.git}}

\item \looseness-1 Experiments (in \Cref{sec:expers}) estimate completeness and soundness scores for various methods.
Our simple method does well on completeness and soundness as one would expect, but the slight surprise is that it performs comparably to prior masked-based methods on earlier saliency evaluations as well, suggesting that soundness can be achieved without paying much in completeness.
  Heuristic choices in existing mask-based methods such as  TV regularization and upsampling, which were hitherto justified extrinsically (i.e., they make masks look natural to humans), receive  an {\em intrinsic} and precise justification: they improve soundness. 
   A theoretical result in \Cref{sec:linear} complements this experimental finding by showing that TV regularization helps saliency methods even in a simple linear classification setting, not just for vision data.
  
\end{itemize}


	\section{Prior approaches}
\label{sec:prior}

We relegate a more thorough description of prior work to Appendix \ref{sec:additionalbackground}, but mention some common saliency and saliency evaluation methods here. Saliency methods aim to explain a model's decision about an input. Saliency evaluation methods aim to evaluate the goodness of a saliency method. 

\textbf{Saliency methods} include \emph{backpropagation based approaches} such as Gradient $\odot$ Input~\citep{shrikumar2017learning},  iGOS++~\citep{khorram2021igos++}, Robust Perturbations \citep{kim2021robust},  LRP~\citep{binder2016layer}, GradCAM~\citep{selvaraju2016grad}, Smooth-Grad \citep{smilkov2017smoothgrad}. Another line of work is \emph{masking methods} which include techniques based on averaging over randomly sampled masks \citep{petsiuk2018rise}, optimizing over meaningful  mask perturbations \citep{fong2017interpretable}, and real time image saliency using a masking network \citep{dabkowski2017real}. Pixels that have been removed from the image by the mask may be replaced by greying out, by Gaussian blurring as in \citep{fong2017interpretable}, or with infillers such as CA-GAN \citep{yu2018generative} used in \citep{chang2018explaining,phang2020investigating}, or DFNet \citep{hong2019deep}. The pixel replacement strategy we used is closely related to hot deck imputation \citep{rubin1976inference}, where features may be replaced either by using the mean feature value (analogous to replacing with grey) or sampling from the marginal feature distribution (analogous to replacing with image pixels sampled from other training images). Some prior work \citep{carter2019made} has found that mean imputation does not significantly affect model output on the beer aroma review dataset. On Imagenette, by contrast, we found that replacement strategy can matter (See Figure \ref{fig:scatter_completeness_soundness}). \citet{de2020decisions} find masks using differentiable masking. 
\emph{Boolean logic} is another approach for saliency methods \citep{ignatiev2019relating,ignatiev2019abduction,macdonald2019rate,mu2020compositional,zhou2018interpreting}. 

\myparagraph{Arguments about saliency.}  Discussions about the methods and the meaning of saliency appear, among others, in  \citep{seo2018noise,fryer2021shapley,gu2018understanding,sundararajan2020many}. \revise{Some reveal situations where Shapley axioms work against feature selection or where Shapley values may be calculated in conflicting ways \citep{fryer2021shapley, sundararajan2020many}. Others question efficacy of saliency methods that add noise~\citep{seo2018noise}, making explanations  non class discriminative \citep{gu2018understanding}.  }

\myparagraph{Saliency evaluation methods.} Extrinsic evaluation metrics include the {\bf WSOL} metric, and  {\bf Pointing Game} metric proposed by \citet{zhang2018top} and {\bf ROAR}  \citep{hooker2019benchmark}. Other more intrinsic methods include early saliency evaluation techniques like {\bf MorF} and {\bf LerF} \citep{samek2016evaluating},  {\bf Insertion and Deletion game} proposed by \citet{petsiuk2018rise}, which involve either inserting pixels in order of most importance or deleting pixels in order of most importance.  {\bf BAM} \citep{yang2019benchmarking} creates saliency maps by pasting object pixels from MSCOCO \citep{lin2014microsoft}.  The  {\bf Saliency Metric} proposed by \citet{dabkowski2017real} thresholds saliency values above some $\alpha$ chosen on a holdout set, finds the smallest bounding box containing these pixels, upsamples and measures the ratio of bounding box area to model accuracy on the cropped image, $s(a,p) = \log(\max(a,0.05)) - \log(p)$ where $a$ is the area of the bounding box and $p$ is the class probability of the upsampled image. 
	\citet{adebayo2018sanity}
proposed ``sanity checks'' for saliency maps. They note that if a model's weights are randomized, it has not learned anything, and therefore the saliency map should not look coherent. They also randomize the labels of the dataset and argue that the saliency maps for a model trained on this scrambled data should be different than the saliency maps for the model trained on the original data. We study the effect of model layer randomization on our method in Appendix section \ref{subsec:sanity} and find that randomizing the model weights does cause our saliency maps to look incoherent. 
\citet{tomsett2020sanity} also discover sanity checks for saliency metrics, finding that saliency evaluation methods can yield inconsistent results. They evaluate saliency maps on reliability, i.e. how consistent the saliency maps are. To measure a method's reliability, because access to ground truth saliency maps are not available, they use three proxies 1) inter-rater reliability, i.e. how whether a saliency evaluation metric is able to consistently rank some saliency methods above others, 2) inter-method reliability, which indicates whether a saliency evaluation metric agrees across different saliency methods, and 3)internal consistency reliability, which measure whether different saliency methods are measuring the same underlying concept.

\myparagraph{Discussions.}
\revise{\citet{brunke2020evaluating} show that perturbation methods are sensitive to baseline and \citet{petsiuk2018rise} point out that human centric explanations (based on bounding boxes) may not reveal why the model made a certain decision. Our notion of intrinsic saliency method outputs the mask but introduces soundness to validate it. }


\section{Masking explanations and completeness/soundness}
\label{sec:stat}

The goal of a saliency method is to produce explanations for the outputs of a given deep net $f$.
We think of a saliency method (\Cref{def:heatmap}) producing explanations --- a heat map in this case --- as providing a {\em proof} of the statement {\em $f$ outputs the label $a$ for input $x$}.  Inspired by logical proof systems, we propose that intrinsic evaluations for saliency methods should test for both completeness and soundness, and the rest of the section provides definitions for these notions.

\subsection{Saliency methods and AUC metric}
\label{sec:methods_metrics}

For a net $f$, let $f(x, a)$ denote the (post-softmax) output probability of the net on input $x$ for label $a$.
We are interested in saliency methods that output a heatmap with a scalar per coordinate of the input.
\begin{definition}[Saliency Method]
\label{def:heatmap}
A saliency method $\method$ is an algorithm that takes 1) a deep net $f$, 2) an input $x$, and 3) a label $a$   and produces a heatmap $\map = \method(f,x,a) \in \mathbb{R}^{\text{dim}(x)}$.
\end{definition}
Note that in the above definition, $\method$ can access any part of $f$ including intermediate layers and gradients.
Examples of methods that can be interpreted as returning heatmaps are mask-based methods \citep{fong2017interpretable,dabkowski2017real,chang2018explaining,agarwal2020explaining,phang2020investigating} that output heatmaps in $[0,1]^{\text{dim}(x)}$, backpropagation-like methods \citep{binder2016layer,selvaraju2016grad}, Shapely values \citep{lundberg2017unified,yeh2020completeness}, and Gradient $\odot$ Input.

\looseness-1 Given such a saliency method, an important question is how to evaluate the quality of the ``explanations'' or heat maps it produces.
A common evaluation idea is to use the behavior of the net $f$ on a ``masked'' input --- only the top few pixels based on the heat map in question are retained, while the others are hidden appropriately.
Inspired by this idea we define evaluation metrics for saliency methods.
We abstract the step of retaining the top few pixels into an {\em input modification process}, defined below.

\begin{definition}[Input Modification Process $\Gamma$]
Let $\Gamma$ be a potentially randomized procedure that takes an input $x$ and a heatmap $\map$ and generates a modified input $\tilde{x}$. i.e. $\tilde{x} \sim \Gamma(x, \map)$.
\end{definition}
This input modification process defines a {\em \textbasemetric} that is the starting point of the completeness and soundness metrics that will be defined later.
\begin{definition}[\textBaseMetric]
\label{def:basemetric}
Let $x$ be an input (image), $a$ a label, $\map$ a heatmap and $\Gamma$ an input modification process.  Then, a {\textbasemetric} is a function $g(x, a, \map) = \mathbbm{E}_{\tilde{x} \sim \Gamma(x, \map)} [f(\tilde{x},a)]$, where $f$ is the neural net of interest\footnote{In general we can replace $f$ with any target function $h$ that depends on $f$}.
\end{definition}
The {\textbasemetric} measures the expected output of $f$ acting on \emph{modified} inputs $\tilde{x}$ and label $a$ and different choices for $\Gamma$ lead to different \textbasemetric s.
Some examples of $\Gamma$ that stay in the {\em intrinsic} framework are: graying out the pixels outside some subset $S$ of $\map$, replacing them by pixels from a Gaussian blurring of $x$~\citep{fong2017interpretable}. 
Input modification through $\Gamma$ also shows up in saliency methods, including our proposed method in \Cref{sec:procedures}; we discuss prior choices of $\Gamma$ in that section.
For our evaluation metrics, however, we choose $\Gamma$ that grays out remaining pixels.

The popular Area-Under-the-Curve (AUC) evaluation metrics~\citep{petsiuk2018rise} of saliency methods can, in fact, be reinterpreted in terms of our {\textbasemetric}.
We state the {\em insertion game} and its AUC below, since it forms the basis of our subsequent completeness and soundness metrics, and then define the AUC metric as an instantiation of \Cref{def:basemetric}.
\begin{definition}[AUC of insertion game]
\label{def:AUC}
    For $s=1$ to $\text{dim}(x)$ take the top $s$ pixels as per $\map$, and plot the probability $f(x, a)$ given by model to label $a$ on the input $x$, where the top $s$ pixels of $x$ are retained and remaining pixels are assigned a default gray value. Return the area under the curve.
\end{definition}
\begin{definition}[AUC metric]
\label{def:AUCmetric}
    We denote by $\AUC(x, a, \map)$ the {\textbasemetric} when a modified input $\tilde{x} \sim \Gamma(x,\map)$ is obtained as follows:
    \begin{itemize}
    \setlength\itemsep{0em}
        \item Sample $s\sim\{1, \dots, \text{dim}(x)\}$ and pick $S$ to be top $s$ pixels as per $\map$.
        \item Set $\tilde{x}[S] = x[S]$ and $\tilde{x}[\bar{S}] = x_{\text{gray}}[\bar{S}]$ where $x_{\text{gray}}$ is a ``default'' input with gray values.
    \end{itemize}
\end{definition}
We note that $\AUC$ is indeed equivalent to the AUC insertion game from \Cref{def:AUC}.
While many choices of {\textbasemetric}s can be used, we use this AUC metric for subsequent definitions since it confers two benefits: 1) it allows for fair comparison between methods since all methods are evaluated at the same sparsity level; 2) it considers the effect of multiple sparsity levels instead of picking an arbitrary one.
We are now ready to present the formalization of completeness and soundness.

\begin{figure}
    \centering
    \vspace{-0.1in}
    \includegraphics[width=\textwidth]{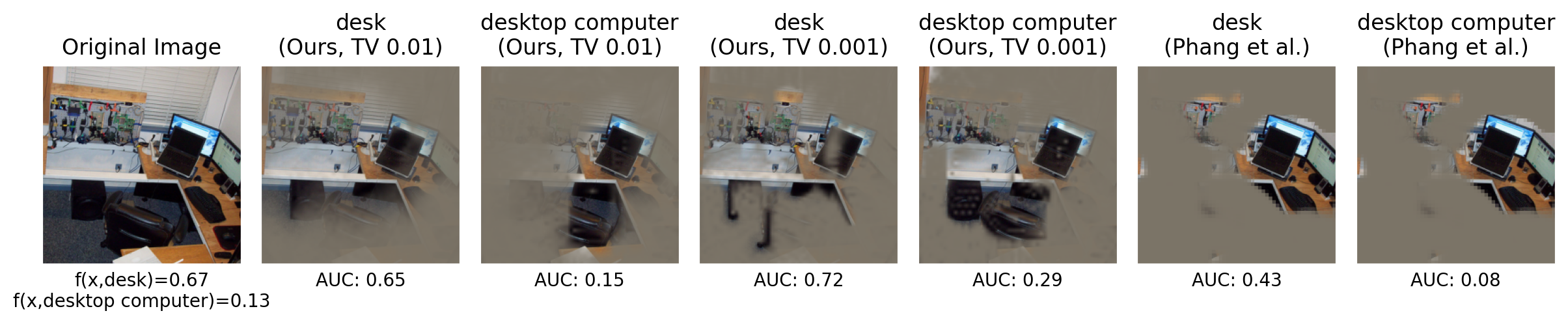}
    \caption{ImageNet image and masked versions for labels desk and desktop computer. Heatmaps generated by our method with $\lambda_{TV}=0.01$, $\lambda_{TV}=0.001$ and \citet{phang2020investigating}'s method, respectively.}
    \label{fig:desktop}
\end{figure}

\subsection{Completeness and soundness}\label{sec:completeness_and_soundness}

Just as for proof systems, we would like a saliency method to be logically complete and sound, i.e., it should be able to justify the label output by the net, but no other label\footnote{We focus on single label multi-class classification; extensions to multi-label are left for future work.}. 
The statements to be ``proved'' are of the form {\em ``does $f$ think $x$ has label $a$?''}.
We use the model output probability $f(x, a)$ as fractional truth value of the ``statement'' $(x, a)$.
For a saliency map $\map$ that a saliency method $\method$ generates, we interpret the AUC metric $\AUC(x, a, \map)$ as the evaluation of truth value of  the ``proof'' $\map$, as adjudged by the {\textbasemetric} $\AUC$.
Thus in this context, completeness for $\method$ would entail that whenever $f(x, a)$ is large (i.e., the statement is sufficiently true), $\AUC(x, a, \map)$ should be large (a sufficient good proof can, and is, generated).
Soundness, on the other hand, would entail that $\AUC(x, a, \map)$ should be small whenever $f(x, a)$ is small.
Thus we would like $\AUC(x, a, \map)$ to track the value of $f(x, a)$ through a lower bound and upper bound as follows:

\begin{definition}[Completeness and Soundness] \label{def:soundcompletestat}
      Consider a saliency method $\method$.
      For $\alpha, \beta \le 1$, the method $\method$ is  {\em $\alpha$-complete on $f, x, a$} if the heatmap $\map = \method(f, x, a)$ it produces satisfies \framebox{$\AUC(x, a, \map) \ge \alpha~ f(x, a)$}, while $\method$ is {\em $\beta$-sound on $f, x, a$} if \framebox{$\AUC(x, a, \map) \le \frac{1}{\beta} f(x, a)$}.
      Conversely the completeness and soundness scores for a heap map $\map$ on $(x,a)$ are defined as:
	\begin{align} \label{eq:cs}
	    \alpha(x, a)
	    = \min \left\{ \frac{\max\{\AUC(x, a, \map), \epsilon_{1}\}}{f(x, a)},  1\right\},~~
	    \beta(x, a)
	    = \min \left\{\frac{\max\{f(x, a), \epsilon_{2}\}}{\AUC(x, a, \map)}, 1\right\} .
	\end{align}

\end{definition}
\begin{definition}[Worst case completeness and soundness]
\label{def:worst_cs}
    For a saliency method $\method$, these metrics are defined by taking the worst case of the completeness/soundness over labels $a$ for a given $x$ that is sampled from the underlying input distribution.
    More precisely,
    $
        \alpha(\method) = \ex_{x}\left[\min_{a} \alpha(x, a)\right],
        \beta(\method) = \ex_{x}\left[\min_{\beta} \beta(x, a)\right]$.
\end{definition}

Our metrics implicitly require a saliency method to output a meaningful map for every label (or at least all labels assigned non-negligible probabilities by the net) and not just for the top model prediction.
Notice, $\alpha$-completeness means that if the model outputs a high probability for label $a$, then the probability for label $a$ after seeing only the coordinates in the salient sets is also high, while $\beta$-soundness means that this probability is not too high. Checking both conditions verifies if the AUC score $\AUC(x, a, \map)$ of a map $\map$ is in the interval $\left[\alpha f(x, a), \frac{1}{\beta} f(x, a)\right]$. 
We want $\alpha, \beta$ as close to $1$ as possible. 
This notion is different from the usual AUC insertion game, which just requires $\AUC(x, \hat{y}, \map)$ to be as large as possible for the model prediction $\hat{y} = \argmax_{a} f(x, a)$.
Moreover, our worst-case completeness and soundness metrics require the $\AUC(x, a, \map)$ to roughly {\em match} the model prediction $f(x, a)$, {\em for every pair $(x, a)$},\footnote{Works like Gradient $\odot$ Input and LRP use {\em completeness} in a somewhat different sense. It requires the sum of the coordinate scores in the heatmap to {\em exactly} equal the logit of the label. Our completeness + soundness together try to approximately match the $\AUC$ score to the output probability on this label.}
where the thresholds $\epsilon_{1}, \epsilon_{2}$ in \Cref{eq:cs} help in ignoring the case where the model outputs very small probabilities --- for instance we can ignore any label for completeness if $f(x,a) < \epsilon_{1}$. 

\subsection{Illustration of completeness and soundness on an ImageNet example}
\label{sec:interpretation}
See \Cref{fig:desktop}. For this image $x$, the model assigns reasonably high probabilities to labels $a_{1}=\text{``desk''}$ and $a_{2} = \text{``desktop computer''}$, with $f(x, a_{1}) \approx 0.67$ and $f(x, a_{2}) \approx 0.13$.
We compare the maps computed by three methods --- $\method_{1}$: our mask learning procedure from \Cref{sec:procedures} with TV regularization 0.01, $\method_{2}$: our method with TV of 0.001, and $\method_{3}$: the mask learned by \citet{phang2020investigating}.
The corresponding maps generated are denoted by $\map_{i}, i\in\{1,2,3\}$.
For this discussion we only consider the effect of labels $a_{1}$ and $a_{2}$ on completeness and soundness, and we assume the thresholds satisfy $\epsilon_{1} = \epsilon_{2} = 0$.

\myparagraph{Completeness.}
For the label $a_{1}$, the AUC scores $\AUC(x,a_{1},\map_{i})$ for the three methods are roughly 0.65, 0.70 and 0.43 respectively.
Thus all three maps can ``certify'' the label $a_{1}$ well enough, with $\map_{2}$ doing the best job.
Consequently the completeness scores (\Cref{eq:cs}) for $\map_{1}$ on the pair $(x, a_{1})$ is $\alpha_{1}(x, a_{1}) = \min\{\frac{0.65}{0.67}, 1\} \approx 0.97$; similarly $\alpha_{2}(x, a_{1}) = 1$ and $\alpha_{3}(x, a_{1})= \frac{0.43}{0.67} \approx 0.64$.
Similarly for label $a_{2}$ that is assigned a probability of 0.13, the completeness scores $\alpha_{i}(x, a_{2})$ are 1, 1, and 0.58.
Thus the worst case completeness from \Cref{def:worst_cs}, i.e. $\min_{j\in\{1,2\}} \alpha(x, a_{j})$, for the three methods are 0.97, 1 and 0.58.
Based on completeness $\method_{2}$ seems like the best method; however soundness will paint a different picture.

\myparagraph{Soundness.} Using the same AUC scores we compute the soundness scores, defined in \Cref{eq:cs} as $\beta_{i}(x, a) = \min\{\frac{f(x, a)}{\AUC(x, a, \map_{i})}, 1\}$.
This gives us soundness scores on $a_{1}$ of 1, 0.95 and 1.
The more interesting case is label $a_{2}$, for which the model does not assign very high probability (0.13).
Here the mask from $\method_{2}$ still gets a very high AUC score of 0.29, leading to a soundness score $\beta_{2}(x, a_{2}) \approx \frac{0.13}{0.29} \approx 0.44$.
$\method_{1}$ on the other hand is not as overconfident in its explanation as $\method_{2}$, and gets $\beta_{1}(x, a_{2}) = \frac{0.13}{0.15} \approx 0.89$.
The worst case soundness for three methods ends up being 0.89, 0.44 and 1.

\looseness-1 This examples highlights how completeness and soundness together can give a more nuanced comparisons of saliency maps, including understanding the benefit of TV regularization.
Experiments in \Cref{sec:expers} do a more thorough comparison of many saliency methods on these and prior metrics.

	\section{Procedures for finding masking explanations}
 \label{sec:procedures}


\looseness-1 We propose a very simple method to find saliency maps with good empirical completeness and soundness scores.
Detailed intuitions and implementations appear in \Cref{sec:procedures-details}.
Our method is similar to prior work on mask-based saliency maps and is based on the idea of SSR (Smallest Sufficient Region)\footnote{We do not incorporate any SDR (Smallest Destroying Region) component in our method.}.
In particular for an input $x$ and label $a$, given a network $f$, the main goal is to learn a map (or mask) $\mask\in\{0, 1\}^{hw}$ such that $\ex_{\tilde{x}\sim \Gamma(x, \mask)}[-\log(f(\tilde{x},a))]$ is minimized, i.e. the probability that the network assigns to a modified (or composite) input $\tilde{x}$ is high.
The key difference from prior masked-based methods is our choice of $\Gamma$, which retains the pixels of $x$ corresponding to $\mask$, but replaces the rest of the pixels with values from a {\em randomly sampled} image $\bar{x}$ from the training set\footnote{We think other random image distributions should work too.} $\DataX$, i.e.
\begin{align}
    \tilde{x}\sim\Gamma(x, a) ~\equiv~ \bar{x} \sim \DataX, \tilde{x} = \mask \odot x + (1 - \mask) \odot \bar{x}\label{eq:composite}
\end{align}
Modifications procedures ($\Gamma$) that have been considered in prior work include: graying or blurring out the remaining pixels \citep{fong2017interpretable}, and using a conditional image generative model to fill in the remaining pixels \citep{chang2018explaining,agarwal2020explaining}.
Replacing with random image pixels forces the mask-finding algorithm to solve a more difficult task of having the network predict a high confidence for the label, despite the presence of another image.
This amounts to grafting salient pixels of $x$ on top of a random image, reminiscent of BAM evaluations~\citep{yang2019benchmarking} for saliency methods.
We also note a methodological difference from \citet{dabkowski2017real,phang2020investigating} in that we do not train a separate neural network to output the mask.

\looseness-1 As is standard, we relax the domain of masks $\mask$ from binary $\{0, 1\}^{hw}$ to continuous $[0,1]^{hw}$ parametrized by a sigmoid.
We define a composite input, where the part of $x$ on $\mask$ is superimposed onto a distractor $\bar{x} \sim \DataX$ as $\tilde{x} = \mask \odot x + (1-\mask) \odot \bar{x}$.
We then have the following simple \emph{vanilla} objective, where the $\ell_1$ norm penalty on $\mask$ helps to reduce the size of masks.
\begin{align}
    L(\mask; (x,a)) =&~ \E_{\bar{x}\sim\DataX} \left[-\log(f(\tilde x, a))\right] + \lambda_1 \|\mask\|_1,
    \text{where } \tilde x = \mask \odot x + (1-\mask) \odot \bar{x}.
    \label{eq:vanilla_obj}
\end{align}
\myparagraph{Promoting soundness.}  Tricks used in prior mask-finding approaches were tested for their effect on soundness. 
These include 1) Total-Variation (TV) penalty \citep{fong2017interpretable} and 2) upsampling of the mask from a lower resolution one \citep{petsiuk2018rise} by learning a low-resolution mask at scale $s$, $\mask\in[0,1]^{hw/s^2}$. For example, on ImageNet, $s=1$ corresponds to learning a $224\times224$ mask, while $s=4$ corresponds to learning a $56 \times 56$ mask and upsampling by a factor of $4$. Incorporating 1) and 2) leads to the following \emph{modified} objective  
\begin{align}
     L(\mask; (x,a)) =&~ \E_{\bar{x}\sim\DataX} \left[-\log(f(\tilde x, a))\right]
     +\lambda_{TV} TV(\mask^{\times s}) + \lambda_1 \|\mask^{\times s}\|_1\nonumber \nonumber \\
    \text{where } \tilde x =&~ \mask^{\times s} \odot x + (1-\mask^{\times s}) \odot \bar{x},
    \label{eq:main_obj}
\end{align}
and $\mask^{\times s}\in\R^{hw}$ is obtained by upsampling $\mask$ by a factor of $s\in\{1,4\}$ via bilinear interpolation. 

While the motivation cited for these two commonly employed ``tricks'' is to avoid artifacts \citep{fong2017interpretable}, in our intrinsic framework ``artifacts'' does not have an  interpretation.  What looks like an artifact to a human observer may be relevant to the net's decision-making. Indeed, our experiments show that while TV penalty or upsampling does produce better looking masks, they lead to a drop in the {\em  completeness metric}. But they significantly improve {\em soundness}. This provides an intrinsic justification for the use of such tricks, instead of extrinsic justifications related to human interpretability. 

\section{On benefit of TV regularization in saliency}
\label{sec:linear}


\begin{figure}[!t]
    \vspace{-0.1in}
\begin{subfigure}[b]{.5\textwidth}
	\centering
	\includegraphics[scale=0.4]{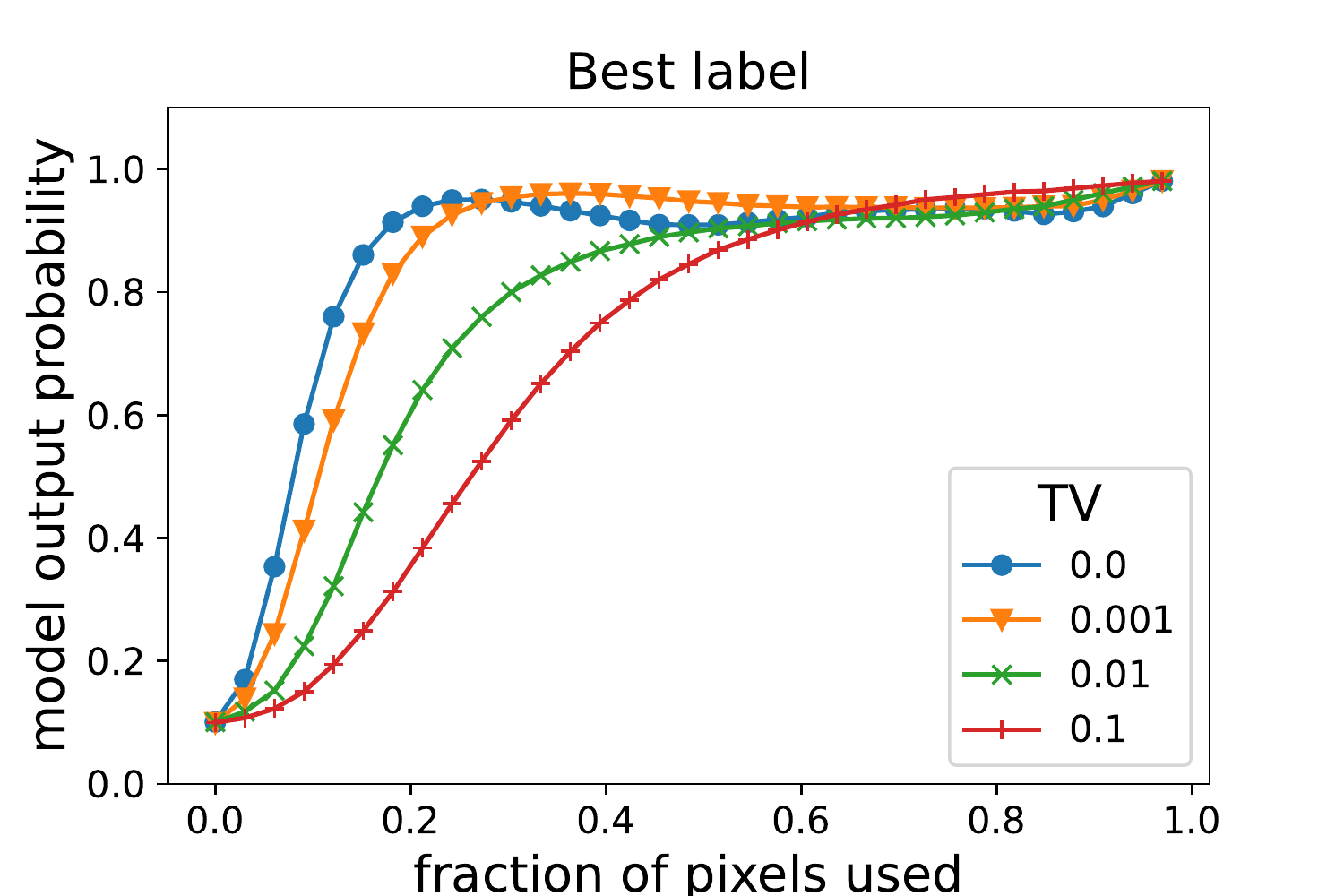}
\end{subfigure}%
\begin{subfigure}[b]{.5\textwidth}
	\centering
	\includegraphics[scale=0.4]{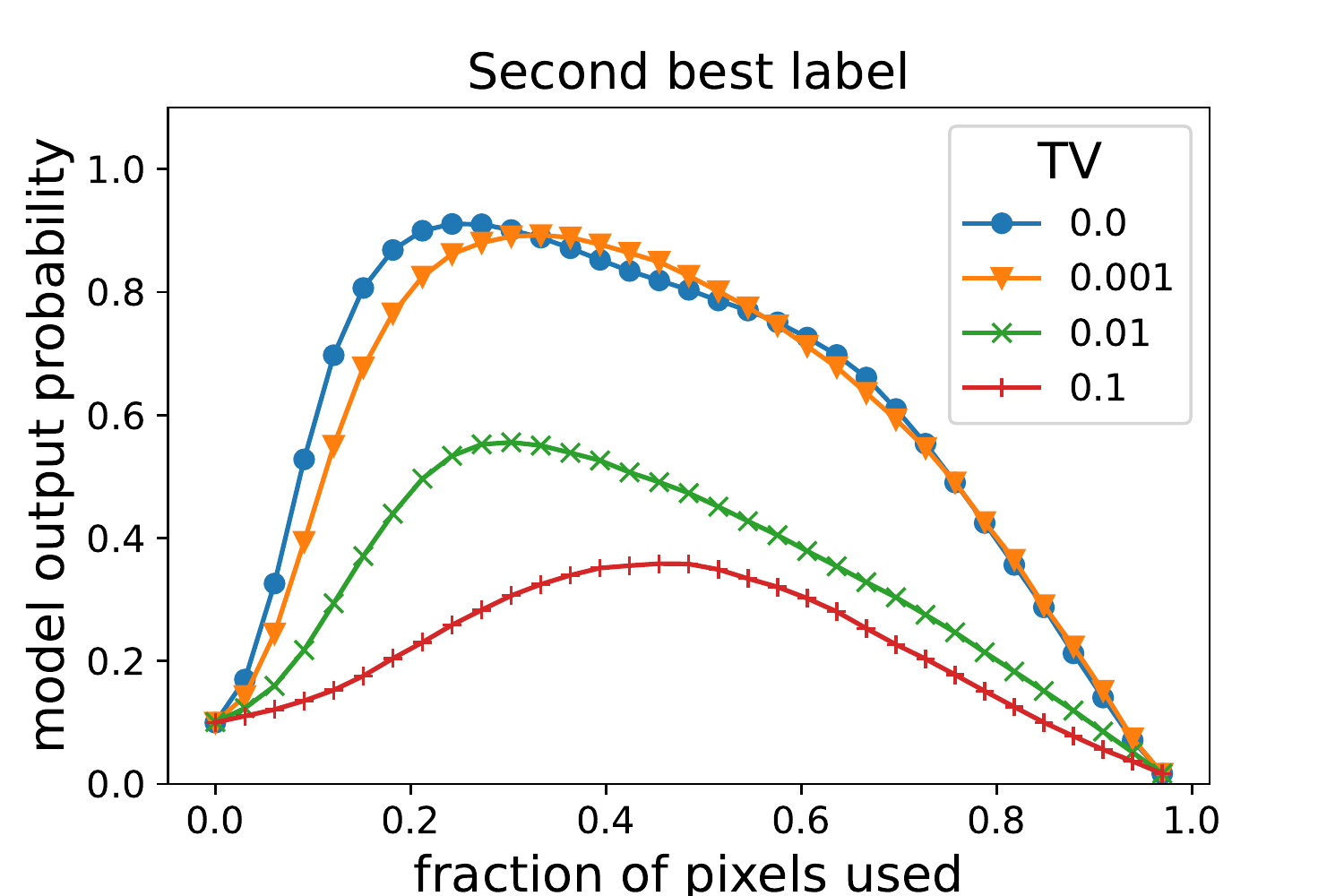}
\end{subfigure}
\caption{Plot of model output probability as more pixels from the original image are retained using learned masks.
The remaining pixels are replaced with gray. Different curves correspond to different values of TV regularization ($\lambda_{TV}$). Larger Area-Under-the-Curve (AUC) for the left figure (best label) suggests good completeness, while lower AUC for the right figure suggests good soundness.
Plots suggest that adding TV significantly helps with soundness, while only slightly hurting completeness.
}
	\label{fig:prob_curves}	
\end{figure}

Past justifications for TV regularization in saliency methods focused on its ability to make heatmaps in images that look more natural to humans. But that is not an intrinsic justification. In the intrinsic framework proposed in this paper, experiments  in \Cref{sec:expers} show TV regularization helps improve soundness of the saliency method. 
The current section reinforces this intrinsic benefit of TV by sketching a simple example (details in \Cref{app:linearcase})  showing how  TV helps even in 
 a simple linear model $\vx \mapsto \vw^\top \vx$ on non-image
input $\vx \in \R^d$ with binary label $y \in \{\pm 1\}$.

Suppose that the classifier has margin at least $\gamma > 0$, that is,
$y \vw^\top \vx = \sum_{i=1}^{d} y w_i x_i \ge \gamma$.
We focus on saliency methods that can return a binary heatmap $\map \in \{0, 1\}^{d}$ to certify the label,
where the coordinates marked as $1$ constitute a set $S$ of saliency coordinates.
Let $\Gamma(x, \map)$ be the input modification process that assigns $0$ to all input coordinates outside the corresponding saliency set $S$.

\looseness-1 
Consider an input $x$ such that $\vw^{\top} \vx > 0$ and $y=1$.
Firstly note that any set $S$ of coordinates that contribute positively to $\vw^\top\vx$,
i.e., the {\textbasemetric} $g(x, y, \map) := \onec{\sum_{i \in S} w_i x_i > 0}$ is $1$, will be able to certify the label $y=1$ and thus would satisfy {\em completeness}.
However this might not be convincing to a user, since it is also possible to come up with a set $S'$ that certifies the label $y=-1$ by picking only negative coordinates of $x$.
This would not satisfy {\em soundness} as the original model only predicts $y=1$.

This issue of soundness can be solved by adding a regularization or constraint on the set of allowable maps, in a way that ensures that it is easy to find a mask $\map$ for the correct label $y$, but not for $-y$.
Ideally, the constraints should be task-specific, but for linear classification with no special structure, even a TV constraint after random permutation suffices.
We consider the constraint that salient sets $S$ should be an interval; this corresponds to $\mathrm{TV}(\map) \le 2$.
The following result shows that if we further constrain the length of intervals to be $\Theta(\frac{1}{\gamma^2} \log d)$, then with high probability, we can find an interval certifying the model prediction $y$ but we are unable to do the same for $-y$
(Details in \Cref{app:linearcase}.) 
\begin{restatable}{theorem}{thmlinear} \label{thm:int1}
	For $(\vx, y) \in \R^d \times \{\pm 1\}$ with $\normtwosm{\vx}=1$, after random shuffling of the coordinates, the following holds for any $L_1 = \Omega(\frac{1}{\gamma^2} \log \frac{1}{\delta})$,
	$L_2 = \Omega(\frac{1}{\gamma^2} \log \frac{d}{\delta})$:
	\begin{enumerate}
		\item (Completeness) With probability $1 - \delta$, there is an interval $\map$ of length $L_1$ s.t.~$g(x, y, \map) = 1$;
		\item (Soundness)  With probability $1 - \delta$, $g(x, -y, \map) = 0$ holds for all intervals $\map$ of length $L_2$.
	\end{enumerate}
\end{restatable}
Although this example is simple, the conceptual message is clear: 
saliency methods may not guarantee soundness in itself,
but adding regularity constraints such as TV can improve soundness.


    

\begin{figure}
    \centering
    \includegraphics[width=\textwidth]{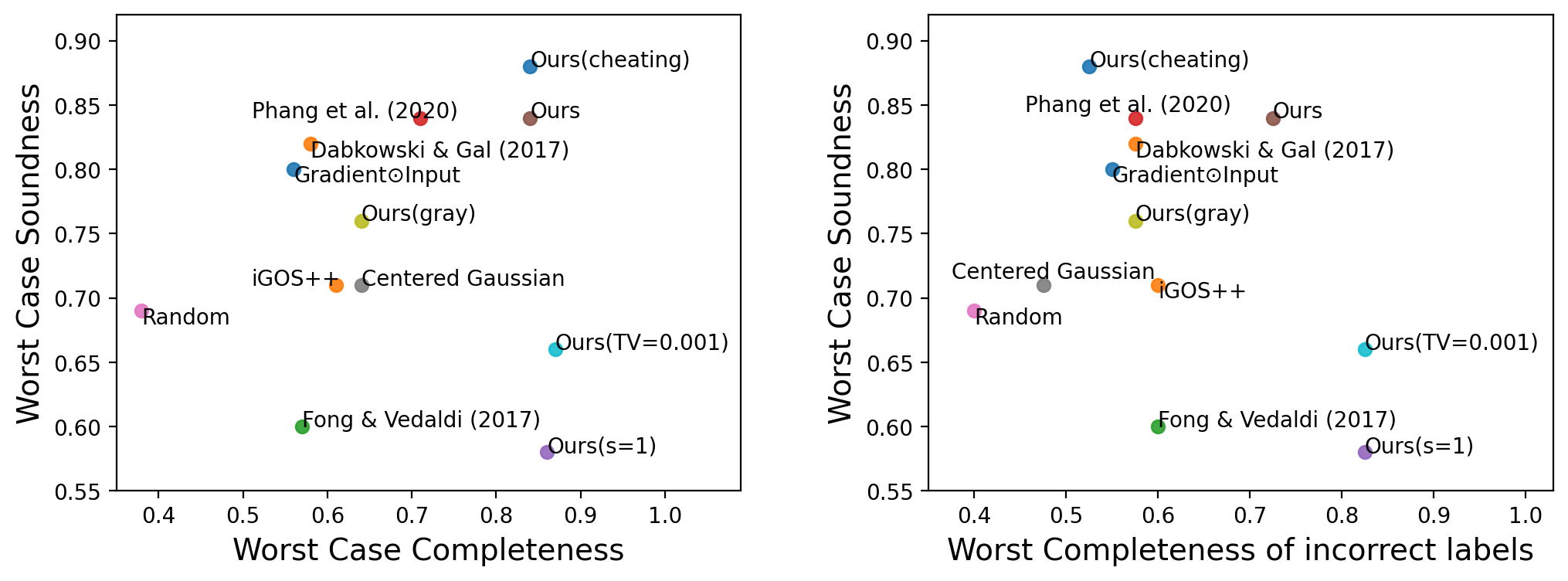}
    \caption{Worst case completeness v.s. soundness plot on Imagenette for different methods. (Upper right corner is best.) 
    Each point represent either a prior method \citep{phang2020investigating,dabkowski2017real,fong2017interpretable,shrikumar2017learning} or our method with different settings. By default, our methods use upsampling $s=4$, TV penalty $\lambda_{TV}=0.001$ and use other images for composite image (see \Cref{eq:composite}). ``Ours'' uses the default setting; ``Ours($s=1$)'' uses upsampling $s=1$; ``Ours($TV=0.001$)'' uses TV penalty factor $\lambda_{TV}=0.001$; ``Ours(gray)'' fills gray pixels in modifications procedures $\Gamma$; and ``Ours(cheating)'' uses masks generated for the correct label as masks for all labels.
    The $x$-axis of the right figure is the average worst completeness of all incorrect labels over samples where the second best label has model probability at least $0.01$. See interpretations of the plot in 
    \Cref{subsec:comp}.}
    \label{fig:scatter_completeness_soundness}
\end{figure}

	\section{Experiments}
	\label{sec:expers}

	 \looseness-1 In this section, we aim to 1) show TV penalty and upsampling have intrinsic justification that they significantly improve soundness at a small cost in completeness; 2) evaluate and compare prior saliency methods, and our method with different parameter settings, on our completeness and soundness metrics.
	 Experiments are performed either on CIFAR-10 or Imagenette \citep{howard2020imagenette}, a 10-class subset of ImageNet \footnote{Completeness and soundness evaluation on all labels can get computationally expensive  on ImageNet as it contains 1000 classes. This cost is only relevant when designing the method; not at deployment time.}. 
	 We compute maps and evaluate metrics on 1000 randomly drawn images from the original test set.
	 Additional experiments on various models and datasets (including ImageNet and CIFAR-100) appear in Appendix \ref{sec:additionalexpers}, as do 
	details of training procedure (beyond those in \Cref{sec:procedures}).
	For all our completeness and soundness evaluations, we use thresholds of $\epsilon_{1} = 0.01$ and $\epsilon_{2} = 0.001$ (see \Cref{eq:cs}).

\subsection{Effect of TV penalty and upsampling}
\label{sec:effect_TV}

In this section, we study the effect of ``tricks'' like TV penalty and upsampling on the learned maps and evaluation metrics.
Figure \ref{fig:prob_curves} shows how soundness and completeness of our masks changes with TV regularization. It plots the model output probability for CIFAR-10 as more pixels from the original image are retained using the mask. Our high level finding is that adding TV penalty significantly aids soundness, while only slightly hurting completeness. Exact evaluation of soundness and completeness for different $\lambda_{TV}$ are listed in \Cref{tab:tv_consistency_cifar10} in \Cref{sec:additionalimagenette}.
Relevant points of \Cref{fig:scatter_completeness_soundness} (which also contain other saliency methods) show effect of upsampling and TV penalty on completeness and soundness. We compare our method from \Cref{sec:procedures} with different settings on Imagenette.
Our method by default use $s=4$ and $\lambda_{TV}=0.01$.
By comparing ``Ours'' against ``Ours($s=1$)'' and ``Ours($TV=0.001$)'', we observe that upsampling and TV penalty leads to slightly worse completeness but significantly better soundness. The effect of upsampling on previous saliency evaluation metric may be seen in the relevant entries of \Cref{tab:imagenette_main}, where it leads to slightly worse performance on the insertion, deletion, and saliency metrics.

    \subsection{Comparison to existing metrics and methods}
    \label{subsec:comp}
	
    \looseness-1 We compared our procedure to other methods including Gradient $\odot$ Input \citep{shrikumar2017learning}, iGOS++\citep{khorram2021igos++}, \citet{dabkowski2017real}  \citet{fong2017interpretable} and \citet{phang2020investigating} on Imagenette. 
    We also include simple baselines like ``Random'' map (each entry is a random Gaussian), and ``Centered Gaussian'' map (2-dimensional isotropic Gaussian distribution placed at the image center).
    Besides completeness and soundness metrics, we evaluate methods on the Deletion Game and Insertion Game metrics (\url{https://github.com/eclique/RISE}), and
	Saliency metric (SM) \citep{dabkowski2017real}.
	Detailed description of the different metrics can be found in the Appendix \ref{sec:additionalbackground}. 
	
	\looseness-1 For our method, we learn a map by optimizing the objective function in \Cref{eq:main_obj} with $\lambda_{TV}\in\{0.01,0.001\}$ and $s \in \{1, 4\}$; default settings are $\lambda_{TV}=0.01, s=4$.
	We also evaluate our method when other pixels are replaced with gray rather than a random image, denoted by ``Ours (gray)''.
    For uniformity, we use the identical ResNet50 pretrained on ImageNet as the base classifier for every saliency method. 
	We normalize the maps so that all values lie in $[0,1]$ before use.

    \begin{table*}[t]
        \centering
        \caption{\looseness-1 Saliency methods evaluate on prior metrics, worst case completeness and soundness on Imagenette. Insertion calculated with grey infilling. The best two of each column are marked bold. By default, our methods use upsampling $s=4$, TV penalty factor $\lambda_{TV}=0.01$ and use randomly sampled image in the modification procedures $\Gamma$. Difference from default setting is noted in parentheses. $\downarrow$~indicates lower is better. $^*$Strong performance of is likely due to centrality bias (cf. \Cref{sec:additionalimagenette}).}
        \label{tab:imagenette_main}
        \resizebox{0.8\textwidth}{!}{
        \begin{tabular}{|c|c c c|c c|}
        \hline
             & Deletion $\downarrow$ & Insertion  & Saliency  & Worst Case  & Worst Case   \\
             &                       & (gray)     $\uparrow$            & Metric $\downarrow$  & Completeness $\uparrow$      &  Soundness $\uparrow$  \\
        \hline
        Gradient $\odot$ Input\cite{shrikumar2017learning} & 0.42 & 0.56 & $-$0.35 & 0.56 & 0.80\\
        iGOS++\cite{khorram2021igos++} & \bf 0.37 &  0.68 & $-$0.36  & 0.61 & 0.71  \\
        \citet{dabkowski2017real} & 0.48 & 0.66 & { $-$0.85} & 0.58 & 0.82\\
        \citet{fong2017interpretable} & 0.58 & 0.59 &$-$0.40 & 0.57 & 0.60  \\
        \citet{phang2020investigating} &  \bf 0.41 & 0.75 & $-$0.27  & 0.71 & \bf 0.84 \\
        \hline
        Ours & 0.50 & 0.74 & { $-$0.90} & 0.84 & \bf 0.84  \\
        Ours ($s=1$) & 0.49 &{\bf 0.77} & {\bf $-$1.18} & \bf 0.86 & 0.58 \\
        Ours ($\lambda_{TV}=0.001$) & 0.45  & \bf 0.80 & \bf $-$1.01 & \bf 0.87 & 0.65 \\
        Ours (gray infilling) & 0.46 & 0.65 & $-$0.35 & 0.64 & 0.76  \\
        \hline
        Random & 0.45 & 0.45 & $-$0.35 & 0.38 & 0.69 \\
        \revise{Centered Gaussian} & 0.66 & { 0.74}$^*$ &  { $-$0.97}$^*$ & 0.64 & 0.71 \\
        \hline
        \end{tabular}
        }
    \end{table*}

\myparagraph{Completeness and soundness.}
We compare different methods on the metrics from \Cref{def:worst_cs}, through the visualization in  \Cref{fig:scatter_completeness_soundness} (left) and in \Cref{tab:imagenette_main}.
We find that our method achieves better completeness and good soundness compared to other methods and baselines.
We posit that completeness is better due to our improved composite input strategy from \Cref{eq:composite}, where non-salient pixels are filled with another image pixels as opposed to other strategies.
As an ablation, we see that ``Ours (gray)'' which uses gray replacement does not perform as well.

\looseness-1On the other hand, our soundness is good only for higher TV regularization $\lambda_{TV}=0.01$ and higher upsampling $s=4$, as discussed in \Cref{sec:effect_TV}.
Soundness is very bad (but completeness slightly better) if these ``tricks'' are not used.
We note that while \citet{phang2020investigating} does as well as our method on soundness, it does so via ``cheating'' in an interesting way: they compute mask for the top label and return the same mask for all labels.
Since the method does not even try to generate masks for labels with low probability, it achieves good soundness for free.
To delve deeper, we employ the same ``cheating'' with our method --- return the same mask for all labels --- and observe a large bump in soundness, without any drop in completeness.
While ``cheating'' is always possible (even if unintentionally) for any metric, we hope that our results inspire saliency method designers to compute maps for all labels.
In this case we find, in \Cref{fig:scatter_completeness_soundness} (right), that looking at worst case completeness of incorrect labels (all except model prediction) gives a measure of the ``effort'' a method employs in generating masks for all labels with probability at least 0.01.
As evident methods that ``cheat'', like \citet{phang2020investigating} and ``Ours (cheating)'', achieve a bad ``best effort'' score.

\myparagraph{Other metrics.}
\Cref{tab:imagenette_main} suggests that our method can achieve comparable performances on most existing intrinsic metrics.
This suggests that good completeness and soundness scores are achievable without compromising on other measures.
Deletion metric is one where our method falls short, likely due to our method not employing an SDR objective as in \citet{dabkowski2017real}.
Although we note that ``Random'' baseline performs quite well on the Deletion metric, which suggests that deletion metric does not provide much signal on these datasets.
Further investigation on the good performance of ``Centered Gaussian'' baseline can be found in \Cref{sec:additionalimagenette}.

\looseness-1 Our settings that works best on other metrics are those without upsampling ($s=1$) or lower TV ($\lambda_{TV}=0.001$).
Soundness is the only metric that strongly justifies the upsampling and TV ``tricks'', thus verifying its utility.
Our method with $s=4,\lambda_{TV}=0.01$ demonstrates that good performance can be obtained on all metrics simultaneously, with only little price to pay.

\myparagraph{Limitations of our work.}
Completeness and soundness do not involve or address human interpretability of explanations. This can be a strength when considering intrinsic evaluations, but is a weakness when considering extrinsic evaluations. Additionally, though completeness and soundness should be taken into account while designing a saliency method, they are not meant as a replacement  for other evaluation methods, which can provide additional information on saliency map quality. For example, the axioms and sanity checks for saliency discussed in Appendix \ref{subsec:missing} should also be applied. Finally, at deployment time, there is a risk that these evaluation methods may be used as a ``proof'' to justify a saliency method is safe to be used in high-risk applications. Completeness and soundness should be taken as a sanity check rather than a proof a saliency method is safe for deployment. 

\vspace{-0.05in}
\section{Conclusions}
\label{sec:conclusions}
\vspace{-0.05in}
Saliency explanations of ML models have proved nebulous and generated much discussion. 
By avoiding extrinsic considerations (e.g., human interpretability) and sticking to intrinsic notions such as completeness and soundness, this paper has tried to give a rigorous notion of saliency  that is intrinsic to the deep net, while avoiding the noisiness in earlier intrinsic ideas that motivated methods like Gradient$\odot$Input and Shapley Values.
Other new contributions include clarifying in the intrinsic view the role of TV regularization and other tricks (it hurts completeness slightly but greatly improves soundness);  a simple saliency method (\Cref{sec:procedures}) for producing mask-based explanations that was designed solely with intrinsic considerations and yet has performance competitive with good existing methods, and sometimes better (\Cref{tab:imagenette_main}). 
Our experiments suggest soundness provides a new dimension to evaluate saliency methods on, and that good performance on it can be achieved without compromising on other metrics.
Note that soundness needs to be considered only while designing and evaluating the method, not at deployment. 

\looseness-1 Evaluating soundness requires looking at maps for all labels instead of just the top label, and thus seems relevant for object {\em localization} in images, and for classification in the wild --- where multiple objects appear in an image.
This benefit is hinted in our evaluation on a prior small dataset (see \Cref{sec:elezeb}) and deserves further exploration.
It may help with certain distribution shifts as well.
While this work studies soundness in the context of mask-based saliency methods and evaluations, the concept of soundness could be applicable for the idea of saliency and interpretability in more generality.

\myparagraph{Acknowledgements.}
We thank Ruth Fong for feedback on an earlier draft of this paper.
We are also grateful to the valuable comments from various anonymous reviewers that helped improve the paper.
This work is supported by funding from NSF, ONR, Simons Foundation, DARPA and SRC.




\bibliography{main}

\newpage

\appendix

\clearpage


\section{Intuitions and Implementations of Procedures to Find Masking Explanations}\label{sec:procedures-details}

As introduced in \Cref{sec:methods_metrics}, for evaluation we may interest in random binary masks due to its connection to AUC, but in our method for finding masking explanations we only focus on deterministic masks.
Given a network $f$, image $x \in \R^{c\times hw}$ and class $a$, we wish to find a binary mask $\mask\in\{0,1\}^{hw}$ such that when the part of $x$ on $\mask$ is superimposed onto a ``distractor'' $\bar{x} \sim \DataX$ (randomly sampled from the train set) as $\tilde{x} = \mask  \odot x + (1-\mask ) \odot \bar{x}$, the output probability of the model $f(\tilde{x},a)$ is high for the class $a$.

As in \Cref{sec:completeness_and_soundness} we compute the average probability assigned to class $a$ over the sampling of the distractor $\bar{x}$, i.e. we are interested in making $\E_{\bar{x}\sim\DataX}[f(\tilde{x}, a)]$ high.
To avoid the hard problem of optimizing over the hypercube $\{0,1\}^{hw}$, a typical strategy (also employed in prior work) is to relax the domain of $\mask $ to be $[0,1]^{hw}$. Since we do not wish to learn masks of very large size, a $\ell_1$ norm penalty on $\mask $ (corresponding to size of the mask), leading to the following natural objective function\footnote{A standard way to maximize probability is to minimize the negative log probability}
\begin{align}
    L(\mask ) = \E_{\bar{x}\sim\DataX} \left[-\log(f(\mask  \odot x + (1-\mask ) \odot \bar{x}, a ))\right] + \lambda_1 \|\mask \|_1
\end{align}

However most masking-based methods employ additional ``tricks'' in order to avoid ``artifacts'' in the produced saliency maps, like Total-Variation (TV) penalty \citep{fong2017interpretable} and upsampling of the mask from a lower resolution one \citep{petsiuk2018rise}.
We also employ the same strategy by learning a low-resolution mask at scale $s$, $\mask \in\R^{hw/s^2}$, to minimize the following
\begin{align}
    L(\mask) = \E_{\bar{x}\sim\DataX} \left[-\log(f(\mask^{\times s} \odot x + (1-\mask^{\times s}) \odot \bar{x}, a))\right] + \lambda_{TV} TV(\mask^{\times s}) + \lambda_1 \|\mask^{\times s}\|_1
\end{align}
where $\mask^{\times s}\in\R^{hw}$ is obtained by upsampling $\mask$ by a factor of $s\in\{1,4\}$ via bilinear interpolation.

While the motivation cited for these ``trick'' is to avoid artifacts, it is not clear whether artifacts are a bad thing, since they might be relevant to the net's decision-making. Indeed, we show that while TV penalty or upsampling does produce better looking masks, they lead to a drop in the {\em  completeness metric}. 
However we show that adding such tricks leads to significant improvement in the {\em soundness metric,} thus providing a novel justification for the use of such tricks, beyond just the heuristic argument of getting rid of artifacts. 
In \Cref{sec:linear} we also provide theoretical justification for why TV penalty can help with soundness, even for the simple case of linear predictors on non-image data.

\looseness-1 We optimize the objective in \Cref{eq:main_obj} by parametrizing $\mask$ as a sigmoid of real valued weights $W\in\R^{hw/s^2}$, i.e. $\mask=\sigma(W)$, and run Adam \citep{kingma2014adam} optimizer for 2000 steps with learning rate $0.05$ and by sampling 10 distractor images at every step, for different values of $\lambda_{TV}$ and upsampling factor $s$.

\section{Practical Benefits of Completeness of all labels for Images of Multiple Objects}\label{sec:elezeb}

Images may have multiple plausible labels. \Cref{fig:elezeb} and \Cref{fig:2label} show images where the classifier net gave high probability to  a single label even though multiple objects were present. Our saliency method can produce different and meaningful masks for all labels that are valid. Extending the notion of soundness for multi-label settings is an open question.

In \Cref{fig:elezeb},  images that previously used in \citealt{gu2018understanding} can have both elephants and zebras present, but it may not be always clear from the model output if there is such a case, since the model can be much more confident on one label, e.g., elephant, than one would expect it to be. For this reason, finding masking explanations validating other labels, e.g., zebra, could provide more information on how the model makes the prediction.

We also use the relabeling provided by \citet{beyer2020we} to select ImageNet validation images with two true labels. Our results may be seen in Figure \ref{fig:2label}.

\begin{figure}[htbp] 
	\centering
	\includegraphics[width=0.45\textwidth]{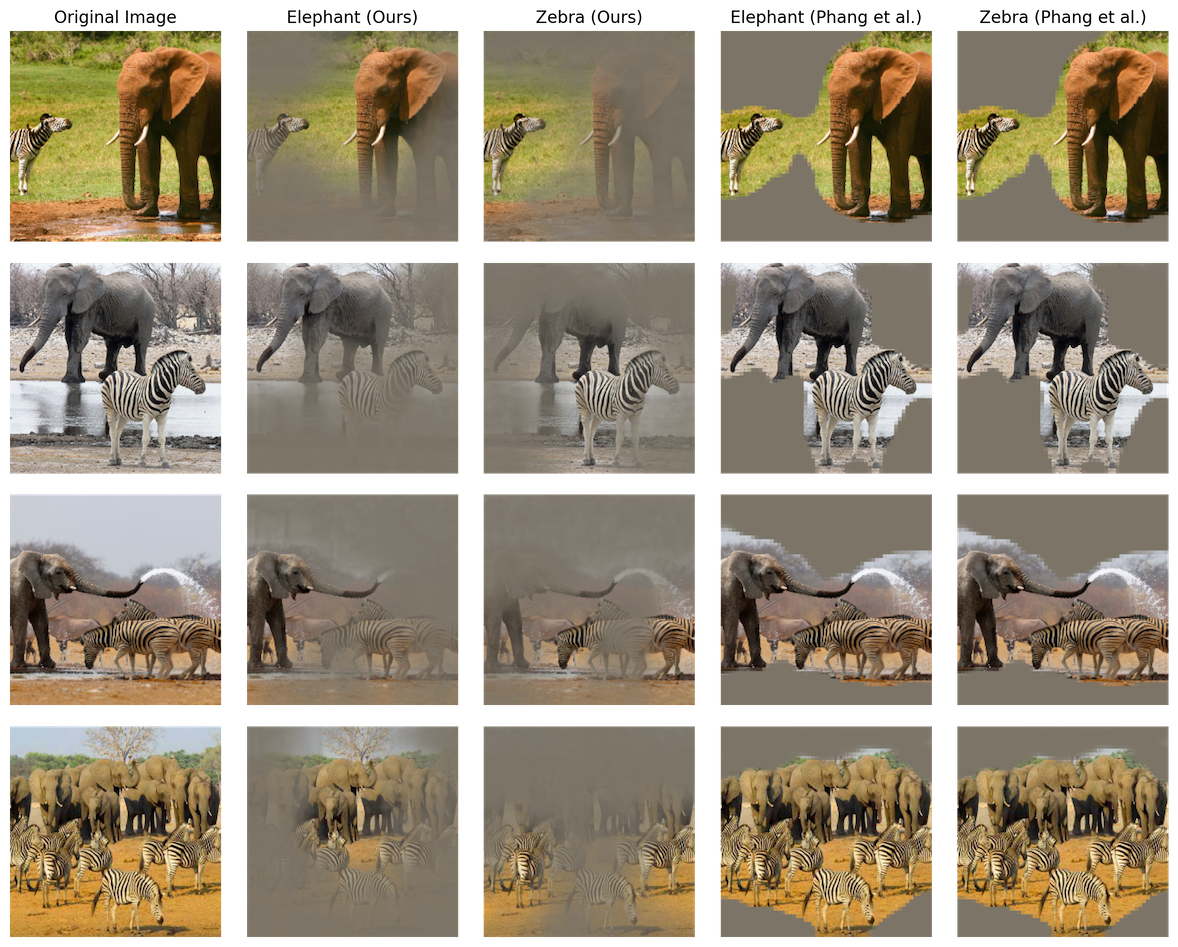}
	\caption{Images containing both elephant(s) and zebra(s), and the corresponding masked ones generated by our method and the best-performing CA model in \citet{phang2020investigating}. The masks by \citet{phang2020investigating} are identical for different labels, and contains both elephant and zebra. In contrast, our method outputs decent masks for elephant and zebra accordingly. For more examples please see \Cref{fig:2label} in \Cref{sec:additionalimagenette}.  }
	\label{fig:elezeb}
\end{figure}

\begin{figure}[htbp] 
	\centering
	\includegraphics[width=0.6\textwidth]{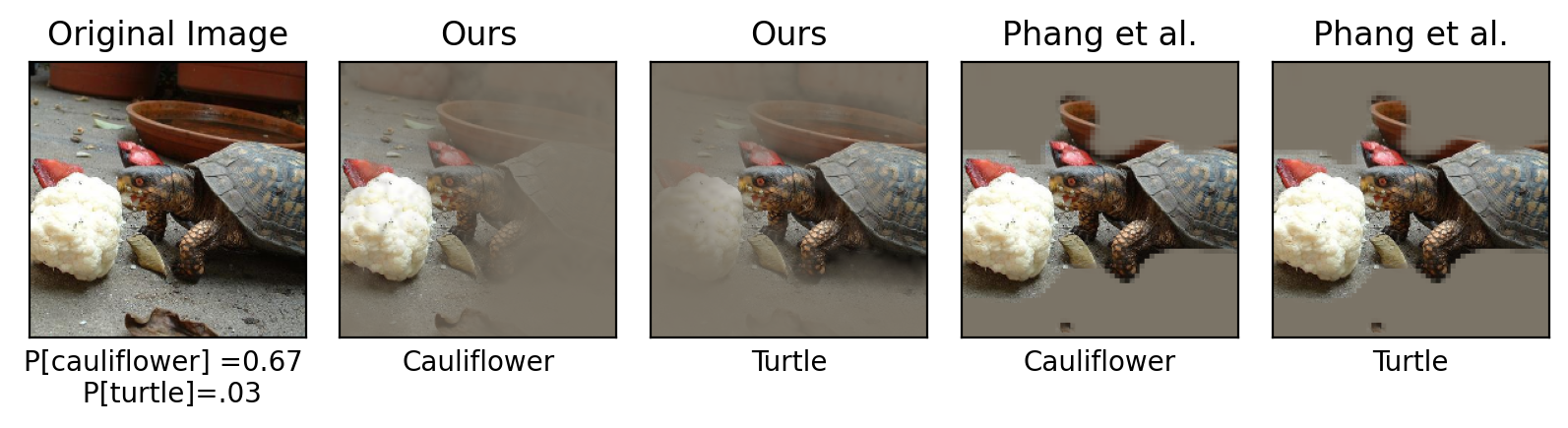}
	\includegraphics[width=0.6\textwidth]{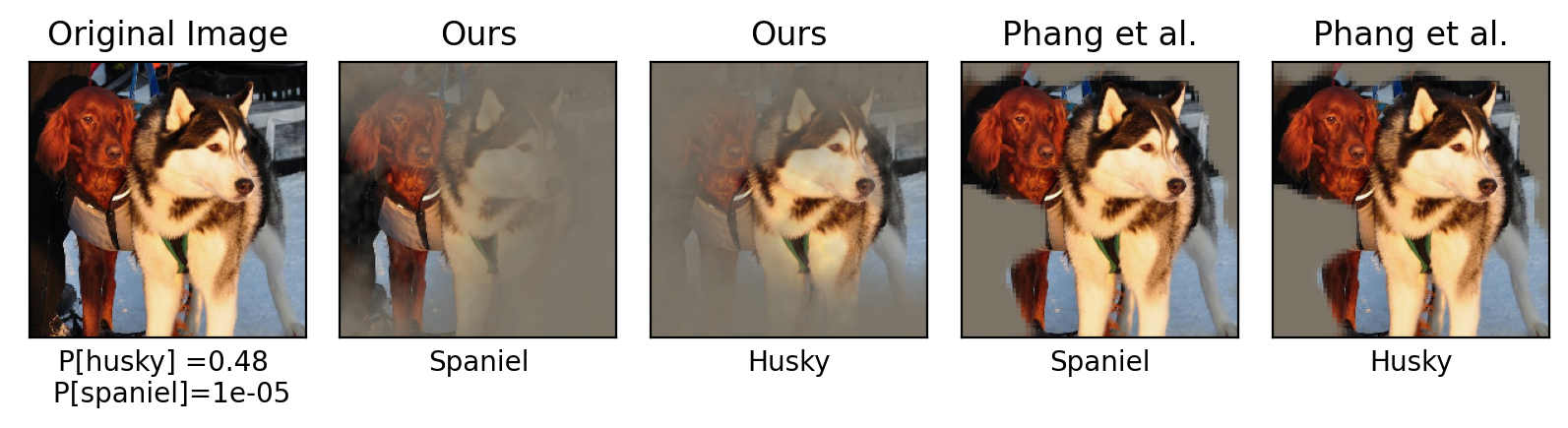}
	\includegraphics[width=0.6\textwidth]{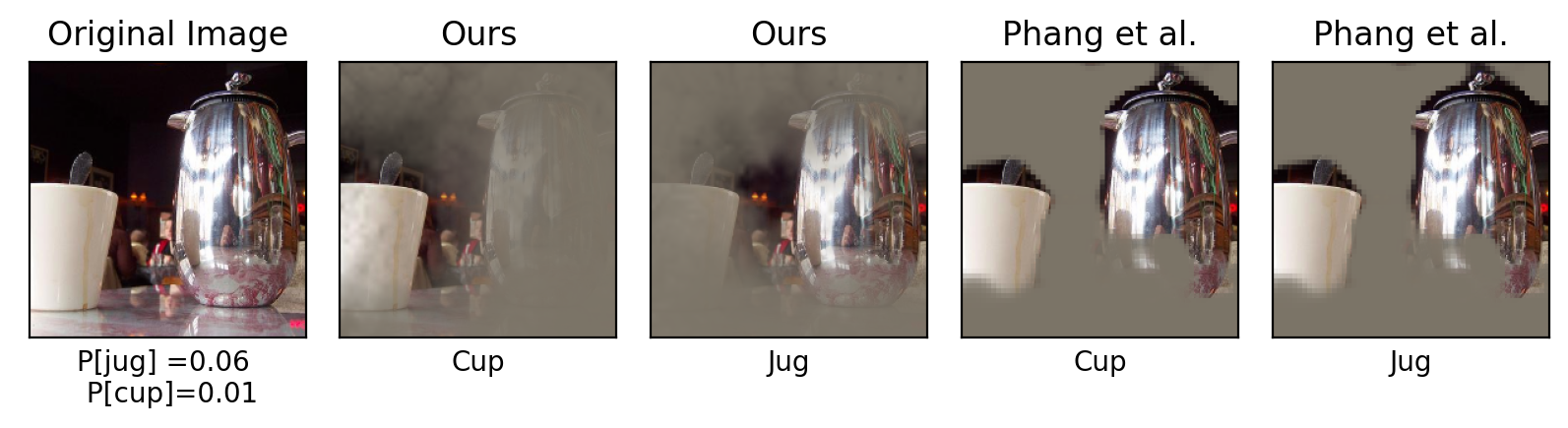}
	\caption{Our masks and masks by~\citet{phang2020investigating} for ImageNet images with two ground truth labels. The best-performing CA model in \citet{phang2020investigating} was used. First (leftmost) column depicts original image with original model probabilities for each ground truth class below the image. Next two columns depict our masks with target label below the image. Final two columns depict the masks for~\citet{phang2020investigating}. The masks by~\citet{phang2020investigating} are identical for different labels, and contains both classes. In contrast, our method outputs descent masks for each class accordingly.  }
	\label{fig:2label}
\end{figure}

\section{Additional Results on CIFAR-10 and Imagenette}\label{sec:additionalimagenette}

\begin{table}[!t]
	    \centering
	    \caption{Worst case completeness and worst case soundness ($\epsilon_1=0.01, \epsilon_2=0.1$) for different $\lambda_{TV}$ in CIFAR-10. The best one in each row is marked bold. }
	    \label{tab:tv_consistency_cifar10}
	    \begin{tabular}{|c|c|c|c|c|}
	        \hline
	        $\lambda_{TV}$ & 0 & 0.001 & 0.01 & 0.1  \\
	        \hline
	        Completeness & \bf 0.99 & 0.99 & 0.89 & 0.80\\
	        Worst soundness & 0.10 &  0.10 & 0.11 & \bf 0.19\\
	        \hline
	    \end{tabular}
	\end{table}

\begin{figure}[htbp]
    \centering
    \includegraphics[height=0.95\textheight]{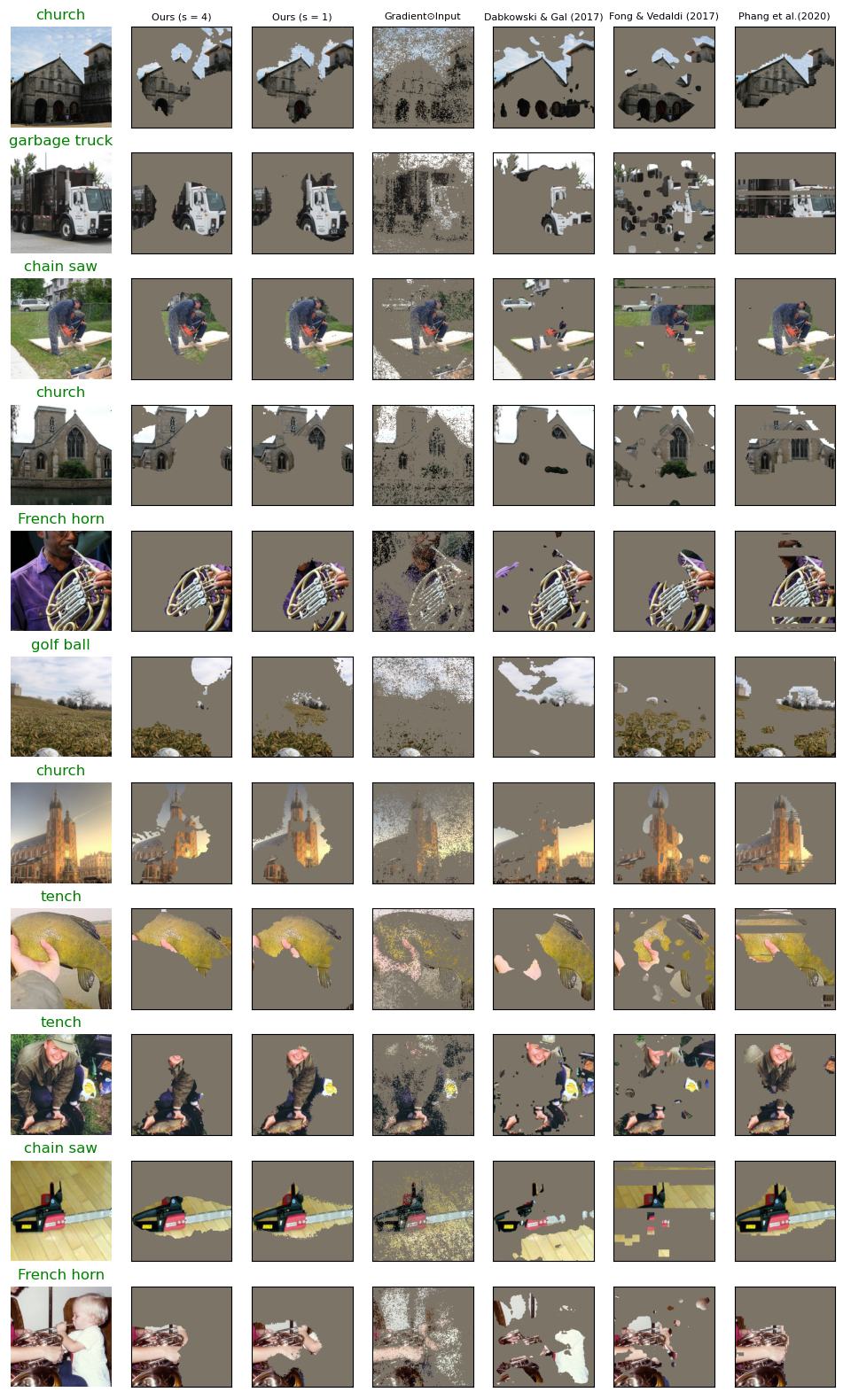}
    \caption{Masks of 10 Imagenette examples generated by different methods. Above each original image is the corresponding correct label. 30\% of the pixels are retained, and the rest pixels are filled with grey.}
    \label{fig:example_masks}
\end{figure}

\begin{figure}
    \centering
    \includegraphics[height=0.95\textheight]{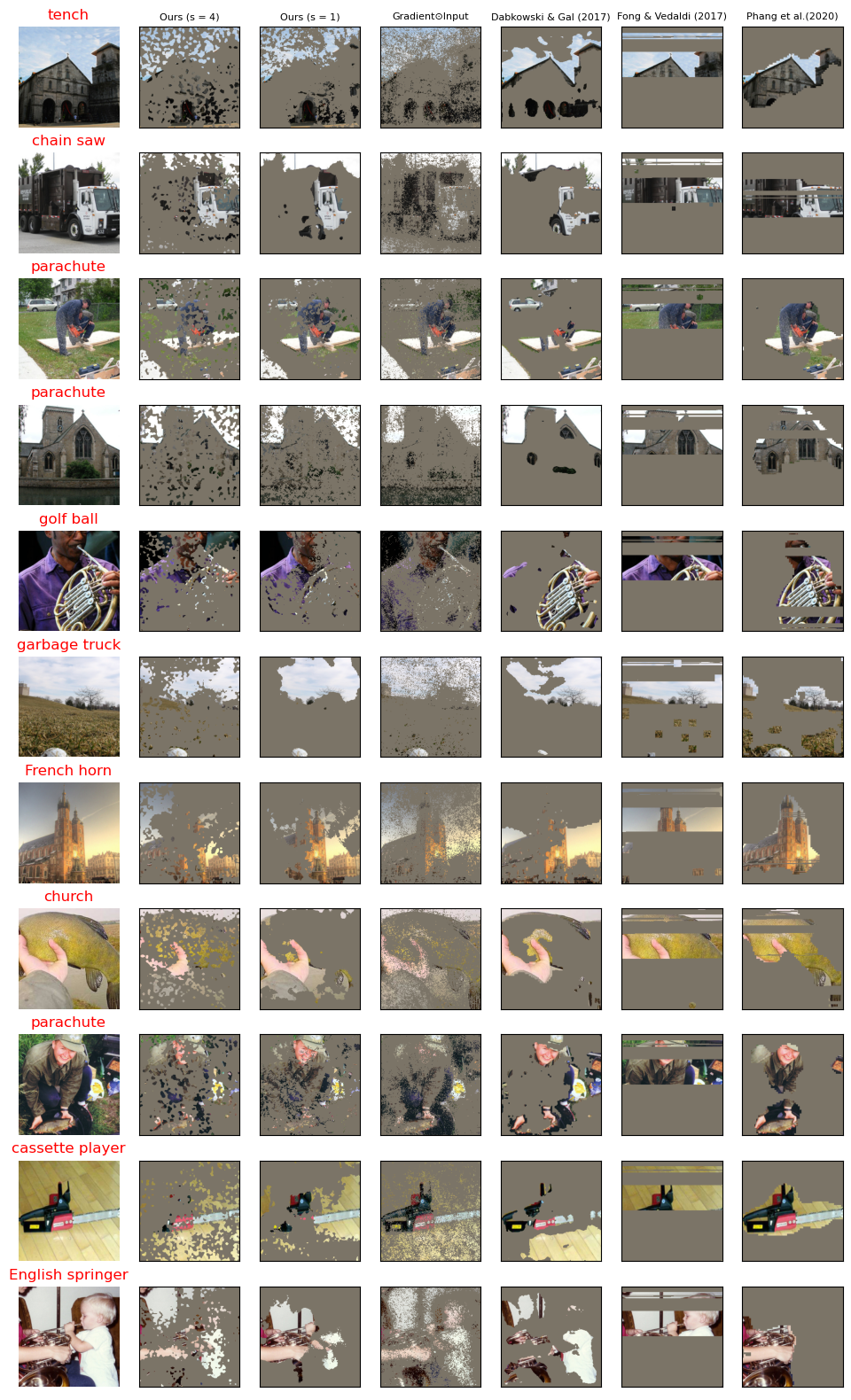}
    \caption{Masks of 10 Imagenette examples generated by different methods for incorrect labels. Above each original image is the  target incorrect label. 30\% of the pixels are retained, and the rest pixels are filled with grey.}
    \label{fig:example_masks_incorrect}
\end{figure}

\begin{table*}[!t]
    \centering
    \caption{Deletion, insertion, saliency metric, worst case completeness and soundness on Imagenette-Corner. Our method with upsampling factor $s=4$, TV penalty $\lambda_{TV}=0.01$ achieves the best completeness and soundness.  ($\uparrow$ indicates higher is better.) The numbers for Centered Gaussian drop compared to Imagenette. }
    \label{tab:imagenette_c_main}
    \begin{tabular}{|c|c c c|c c|}
    \hline
         & Deletion $\downarrow$ & Insertion  & Saliency  & Worst Case  & Worst Case \\
         &                       & (gray)     $\uparrow$            & Metric $\downarrow$  & Completeness $\uparrow$ & Soundness $\uparrow$ \\
    \hline
    Gradient $\odot$ Input & {\bf 0.13} & 0.34 & \bf $-$0.25 & 0.50 & 0.78\\
    \citet{dabkowski2017real} & 0.21 & 0.33 & {\bf $-$0.25} & 0.52 & 0.78\\
    \citet{fong2017interpretable} & \bf 0.15 & 0.53 & \bf $-$0.25 & 0.53 & 0.48 \\
    \citet{phang2020investigating} & { 0.33} & {\bf 0.69} & \bf $-$0.25  & \bf 0.68 & \bf 0.82 \\
    \hline
    Ours & 0.36 & {\bf 0.64} & \bf $-$0.25 & \bf 0.76 & \bf 0.84 \\
    \hline
    Random & 0.35 & 0.35 & $-$0.24 & 0.36 & 0.64\\
    Centered Gaussian & 0.62 & 0.62 &  0.458 & 0.50 & 0.54 \\
    \hline
    \end{tabular}
\end{table*}

In this section, we show additional results for CIFAR-10 and Imagenette experiments. 
\begin{enumerate}[leftmargin=*]
    \item  In \Cref{tab:tv_consistency_cifar10}, we show the worst case completeness and worst case soundness on CIFAR-10 of our method with different TV penalty $\lambda_{TV}$. To showcase of generality of our definition of base metric, here we use a modified version of AUC metric where we set $s\sim \{0.2\dim(x), \ldots, 0.6\dim(x)\}$ in \Cref{def:AUCmetric}. 
    \item We showcase the masks of ten randomly drawn images from Imagenette for different methods in \Cref{fig:example_masks} and \Cref{fig:example_masks_incorrect}. The masks in \Cref{fig:example_masks} is generated for model predictions, while the masks in \Cref{fig:example_masks_incorrect} is generated for incorrect labels. Note there are some wired horizontal lines and shapes for some masks. These are caused by default tie breaking of pixels of masking value 1. We also tried tie breaking according to Centered Gaussian, which does not improve the performance of those methods.
    \item We create a new dataset Imagenette-Corner based on Imagenette, where each image is a random corner of the original image\footnote{On ImageNet or Imagenette, the common way to process image is first resizing it to $256\times 256$ pixels, and then crop the center $224\times 224$ pixels. In Imagenette-Corner, instead of center cropping, we take one of the four $180\times 180$-pixel corners and resize it to $224\times 224$ pixels.}. We show the results on Imagenette-Corner in \Cref{tab:imagenette_c_main}. The results show that the good performance of Centered Gaussian on Imagenette is likely due to the bias of datasets (that objects are centered with high probability). Saliency metric on Imagenette-Corner are the same for most of the methods, likely because of the coarse selection of hyperparameters as in \citet{dabkowski2017real}.
\end{enumerate}




\section{Experimental Details and Additional Experiments}
\label{sec:additionalexpers}
	In this section we expand upon the experiments in \Cref{sec:expers} and complement them with more experiments on the ImageNet, CIFAR-10 and CIFAR-100 datasets.
	For each of the datasets we test the following:
	\begin{itemize}[leftmargin=*]
	    \item {\bf Visualization:} For various values of TV regularization (and upsampling for ImageNet), we visualize the mask and also what part of the image a sparse version of the mask highlights. We do so for masks learned for the correct label and also for the second most probable label as predicted by the model. The common trend is that while TV regularization (and upsampling) make the masks more human interpretable, it also makes it harder to find a good mask for the incorrect label, thus improving soundness.

	    \item {\bf AUC curve:} We plot the output model probability for a masked input as more pixels from the original image are selected. The 4 plots denote replacing remaining pixels with gray pixels or pixels from a random image, and masks to fit the correct or incorrect labels, i.e. most probable and second most probable labels.
	    Again, we find the TV regularization and upsampling help with soundness; i.e. inability to find mask for the second most probable label.
	    For mask $M$, if $\bar{M}(p)$ denotes the discrete mask with top $p$ fraction of the pixels from $M$ picked.
	    We plot $\E_{x} \left[\E_{x'\sim\Gamma}[f(\bar{M}(p)\odot x + (1 - \bar{M}(p)) \odot x', a)]\right]$ v/s $p$, where $\Gamma$ is either a random image or a gray image, and $a$ is either the correct label for $x$ or the second best label.
	    We note that replacing with gray and random image lead to similar looking plots, with the probability estimate of random image being more pessimistic.
	    This justifies the motivation for our procedure that learns a mask to solve a ``harder task'' of random image replacement.

	    \item {\bf Completeness/soundness:} We evaluate the completeness and soundness scores, as defined in \Cref{eq:cs} and \Cref{def:worst_cs}.
	    In particular for any input $x$, we only evaluate the scores for the top model prediction $a$ and the second most probable label $a'$ and report the worst case completeness and soundness for these 2 labels.
	    For all experiments in this section, we use $\epsilon_{1}=0$ and $\epsilon_{2}=0.1$ (from \Cref{eq:cs}).

	    \item {\bf Intrinsic metrics:} We evaluate our masks on other intrinsic metrics from prior work, and compare to baseline saliency methods.
	    Our baselines include Gradient $\odot$ Input \citep{shrikumar2017learning}, Smooth-Grad \citep{smilkov2017smoothgrad}, Real Time Saliency \citep{dabkowski2017real} (for ResNet-50 on ImageNet), and Random indicating a random Gaussian mask as a control. We use Captum \citep{kokhlikyan2020captum} for Gradient $\odot$ Input and Smooth-Grad implementations and the original author code\footnote{\url{https://github.com/PiotrDabkowski/pytorch-saliency}} for Real Time Saliency. When calculating the Saliency Metric (SM) \citep{dabkowski2017real} we tune the threshold $\delta$ on a holdout set of size 100 with $\delta$ between $0$ and $5$ in increments of $0.2$ as in prior work.
	    
	    For the saliency method of \citet{fong2017interpretable} that we only used on the Imagenette, we adapt the most popular implementation on GitHub\footnote{\url{https://github.com/jacobgil/pytorch-explain-black-box}}. The implementation contains minor deviations from the original paper as described on its main page. For \citet{phang2020investigating}, we used their best CA model pretrained and provided in original author code\footnote{\url{https://github.com/zphang/saliency_investigation}}.
	 	\end{itemize}

\subsection{CIFAR-10 Experiments}
\label{subsection:cif10}

\begin{figure}[t!]
\centering
\includegraphics[scale = 0.25]{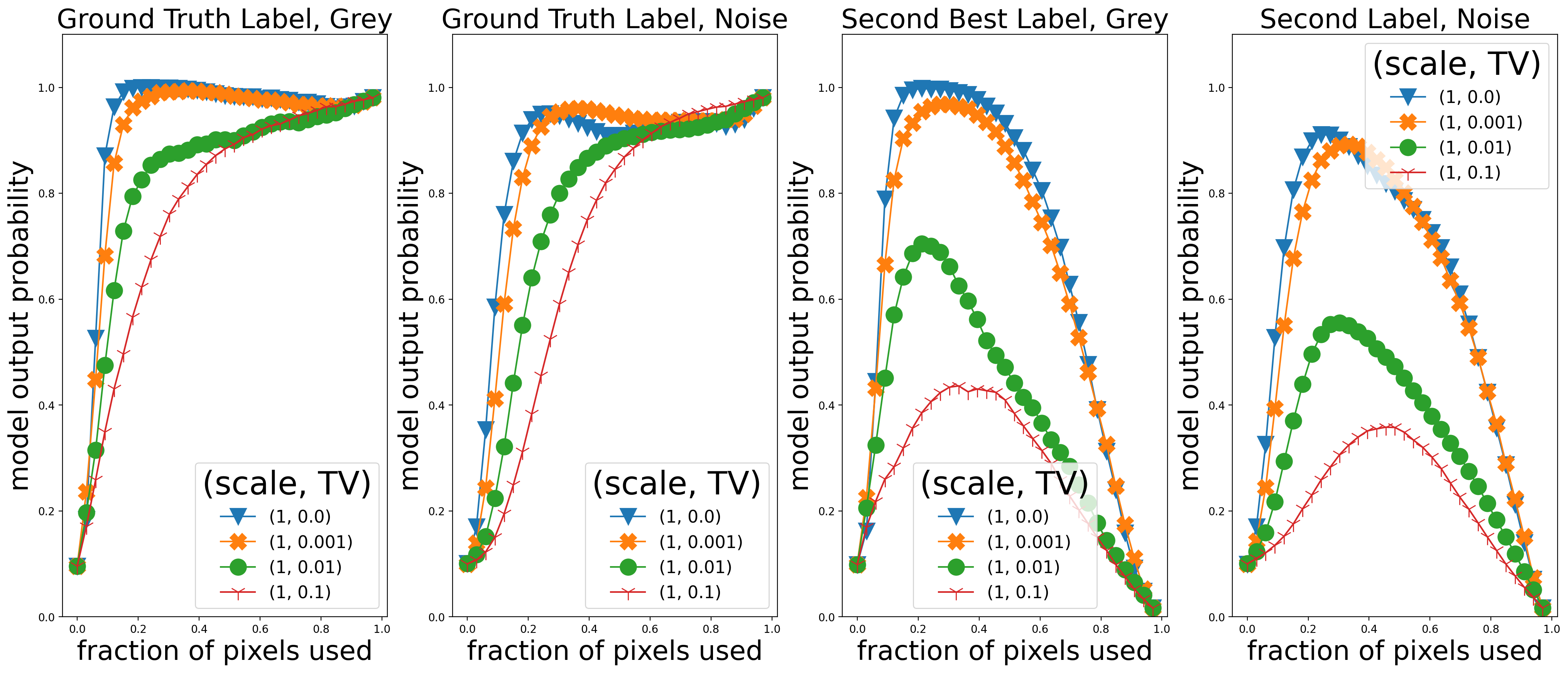}
\caption{[CIFAR-10] AUC curves with as the fraction of pixels retained from the original images based on the mask varies from 0 to 1.0 on the X-axis. The probabilities assigned by the model (averaged over 1600 images) on the Y-axis. {\bf Left: } Mask learned for ground truth label, probabilities for ground truth label while replacing remaining pixels with gray. {\bf Center Left: } Mask learned for ground truth label, probabilities for ground truth label while replacing remaining pixels with other image pixels. {\bf Center Right: } Mask learned for second best label, probabilities for second best label while replacing remaining pixels with gray. {\bf Right: } Mask learned for second best label, probabilities for second best label while replacing remaining pixels with other image pixels. We see that increasing TV regularization results in only a mild drop in completeness, but significantly improves soundness.}
	\label{fig:cif10_auc}	
\end{figure}

\begin{table}[t!]
\begin{center}
	\caption{Completeness and soundness for a ResNet-164 trained model on CIFAR-10, as defined in \Cref{eq:cs}.
	Each column represents mask learned using our procedure, with different TV regularization strengths (0.0, 0.001, 0.01, or 0.1).
	``Gray'' indicates pixels were grayed during AUC calculation and ``Random image'' indicates they were replaced with other images. }
	\label{table:cif10_complete_sound}
	\vspace{0.1in}
\begin{tabular}{ |c|c|c|c| c| } 
 \hline
 Gray & TV = 0.0 & TV= 0.001 & TV= 0.01 & TV= 0.1   \\ 
 \hline
Completeness ($\alpha$)       & 0.92   & 0.91  & 0.83    &  0.77    \\

Soundness ($\beta$)     &  0.17   & 0.18   &  0.39  & 0.57   \\ 
 \hline
 Random image & TV = 0.0 & TV= 0.001 & TV= 0.01 & TV= 0.1   \\ 
 \hline
Completeness ($\alpha$)       &   0.86   & 0.85   &   0.77    &  0.70   \\
Soundness ($\beta$)     &  0.20   & 0.21   & 0.44  & 0.62   \\
 \hline
\end{tabular}
\end{center}
\end{table}

\begin{table}[t!]
\begin{center}
	\caption{Performance of our method on CIFAR-10 and some baselines on various intrinsic saliency metrics proposed in prior work.
	Downarrow (uparrow) means lower (higher) is better.
	We find that while both our masks (learned with and without TV) have very good performance on the insertion metric. The deletion and saliency metrics are uninformative in this case, since all methods are as good (or worse) compared to a random mask.}
	\label{table:cif10}
	\vspace{0.1in}
\begin{tabular}{ |c|c|c|c| c| c| } 
 \hline
 & Gradient $\odot$ Input & Our method & Our Method & Smooth-Grad & Random \\ 
 &  & ($\lambda_{TV}=0.01$) & ($\lambda_{TV}=0$) & saliency &   \\ 
 \hline
 Deletion $\downarrow$          & 0.32    &  0.37   &    0.59 & 0.31 &  0.26  \\

Insertion (blur) $\uparrow$     &  0.60   & 0.88   &    0.94   & 0.66 & 0.36 \\ 
Insertion (gray) $\uparrow$     &  0.51  & 0.83    &    0.92  & 0.55 & 0.26 \\ 
Saliency Metric $\downarrow$    & 0.22   & 0.22    & 0.22   & 0.23 &  0.22 \\ 
 \hline
\end{tabular}
\end{center}
\end{table}

We also run our method from \Cref{sec:procedures} on the CIFAR-10 dataset using a pretrained ResNet-164 architecture\footnote{\url{https://github.com/bearpaw/pytorch-classification}. The ResNet-110 model in this repository is actually a ResNet-164 model.}.
For all experiments we learn a mask $M\in\R^{32\times 32}$, thus using a scaling factor of $s=1$ (no upsampling).
We train masks for 1600 images that were correctly classified by the pretrained ResNet-164 using regularization parameter $\lambda_{TV}\in\{0, 0.001, 0.01, 0.1\}$.
We use a (fixed) L1 regularization value of .001 for all masks.

We visualize the masks learned for the correct label in \Cref{subfig:ciftrue} and in \Cref{subfig:cifsbest} we visualize the same for the second best label predicted by the ResNet-164 model.
We also visualize the masks for all labels for some randomly picked images in \Cref{fig:artifact} to demonstrate the commonness of artifact, especially for the incorrect labels.
The AUC curves in \Cref{fig:cif10_auc} suggest a similar trend to that of Imagenette, adding TV regularization results in only a mild drop in completeness, but significantly improves soundness.
Evaluation of our masks, compared to some gradient baselines, on intrinsic metrics can be found in \Cref{table:cif10}.
We report the completeness and soundness results for CIFAR-10 in Table \ref{table:cif10_complete_sound} for TV values in $(0.0,0.001,0.01,0.1)$ calculated using a ResNet-164 model.

\begin{figure}[th!]
\centering
\begin{subfigure}[t]{.85\textwidth}
\centering
\includegraphics[width=\textwidth]{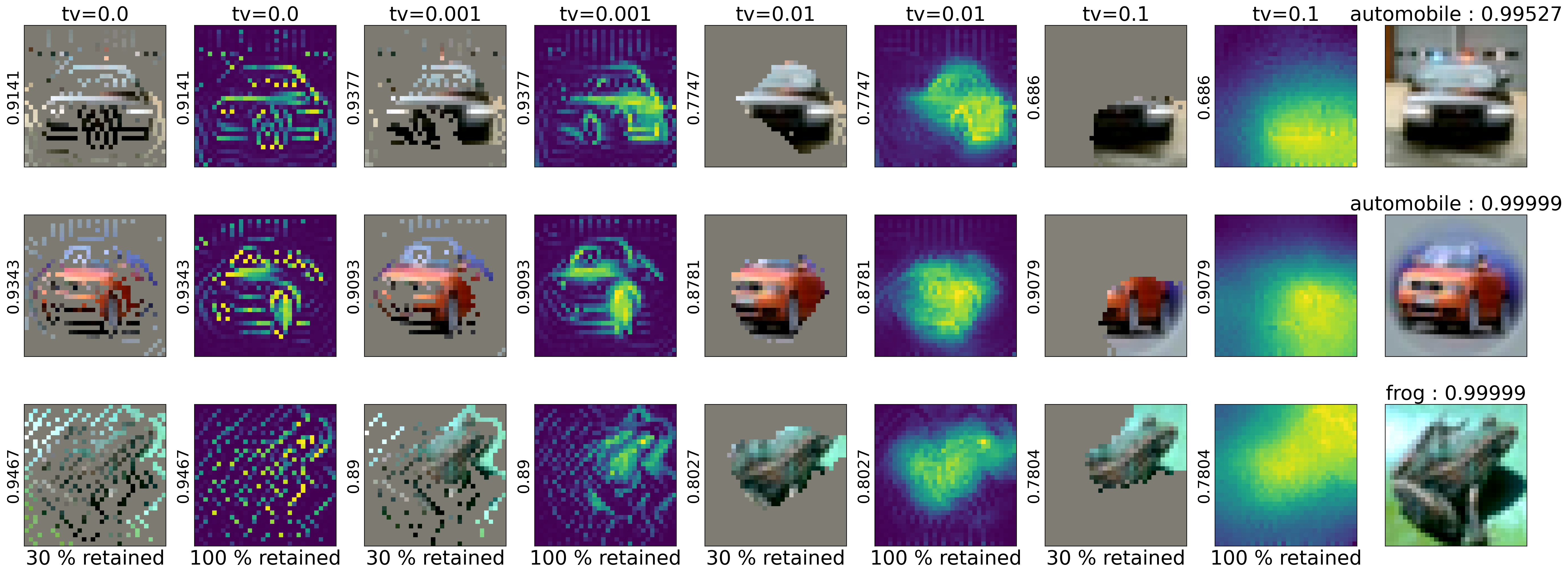}
\caption{}
\label{subfig:ciftrue}
\end{subfigure}
\begin{subfigure}[t]{.85\textwidth}
\centering
\includegraphics[width=\textwidth]{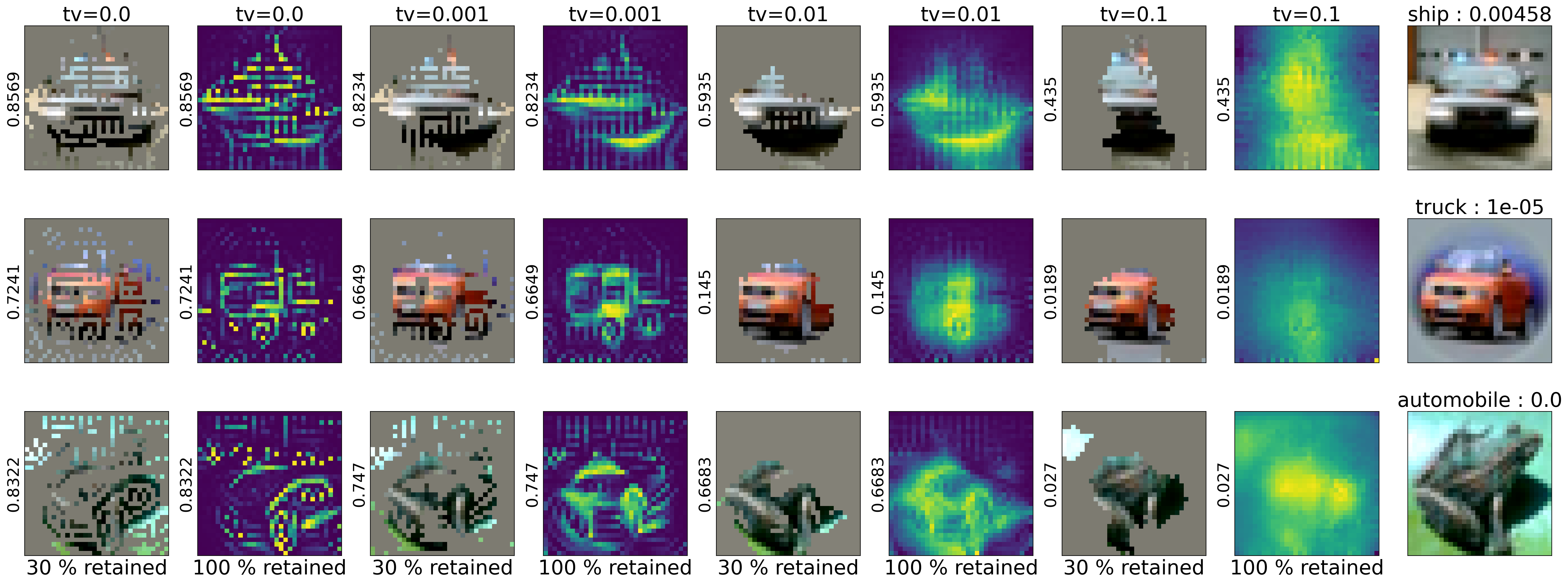}
\caption{}
\label{subfig:cifsbest}
\end{subfigure}
\caption{Details in \Cref{subsection:cif10} {\bf Panel \ref{subfig:ciftrue} } Masks learned for the correct label of CIFAR-10 images using the procedure outlined in \Cref{sec:procedures} with ResNet-164. Columns (1,3,5,7) depict masked images at 30 (retained) \% mask sparseness. Columns (2,4,6,8) depict the original mask. TV values shown above. Original image shown in rightmost column. Model probability of correct label for masked images on y axis. {\bf Panel \ref{subfig:cifsbest}} Masks leared for the second most probable label of CIFAR-10 images using the procedure outlined in \Cref{sec:procedures} on ResNet-164. Columns (1,3,5,7) depict masked images at 30 \% mask sparseness. Columns (2,4,6,8) depict the original mask.TV values shown above. Original image shown in rightmost column. Model probability of second best label for masked images on y axis. }
	\label{fig:cif10_first}	
\end{figure}

\begin{figure} 
\begin{subfigure}{.99\textwidth}
  \centering
  \includegraphics[width=.9\linewidth]{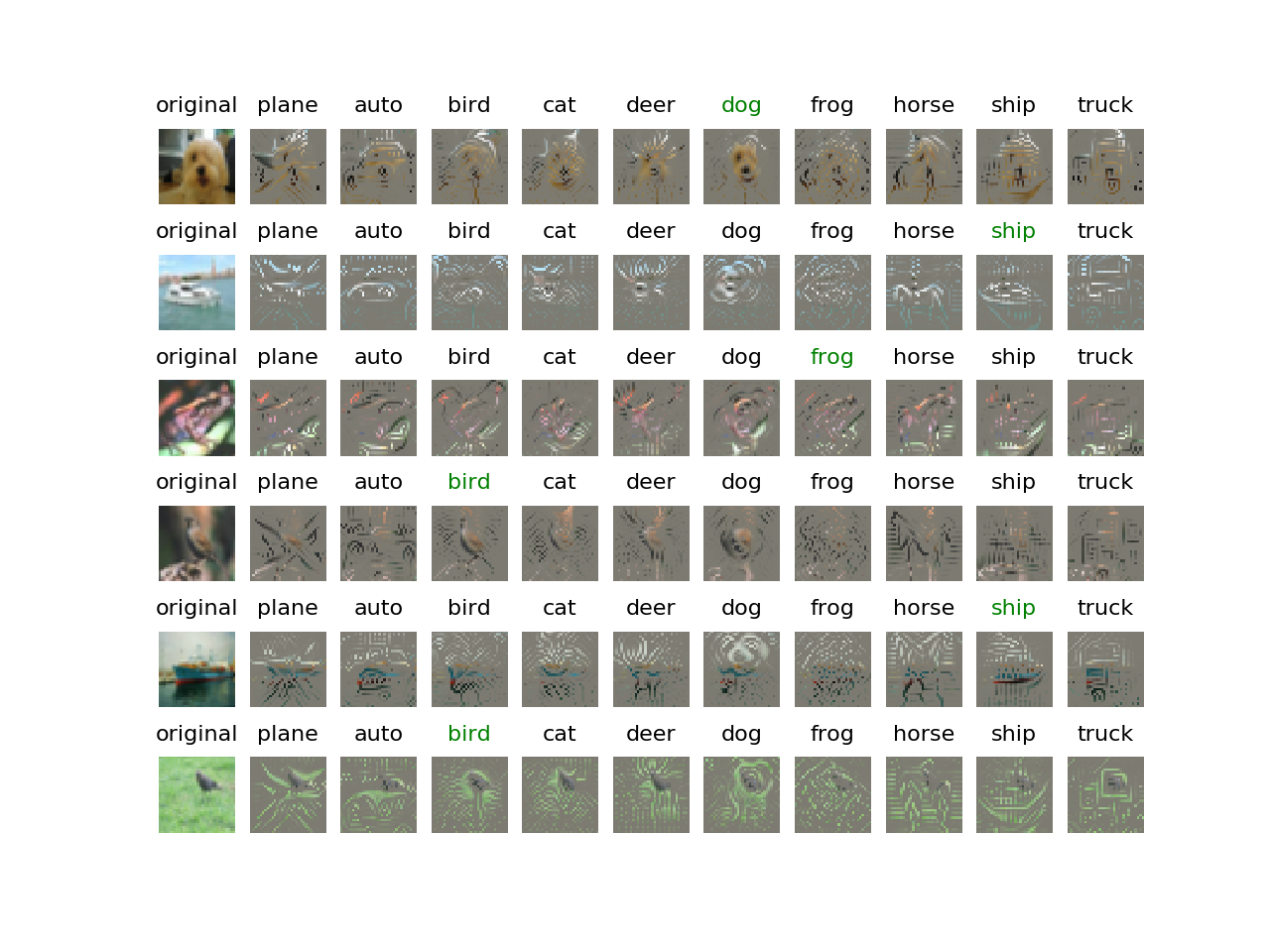}
  \vspace{-0.1in}
  \caption{Our method with no TV regularization}
  \label{fig:artifact-0}
\end{subfigure}

\begin{subfigure}{.99\textwidth}
  \centering
  \includegraphics[width=.9\linewidth]{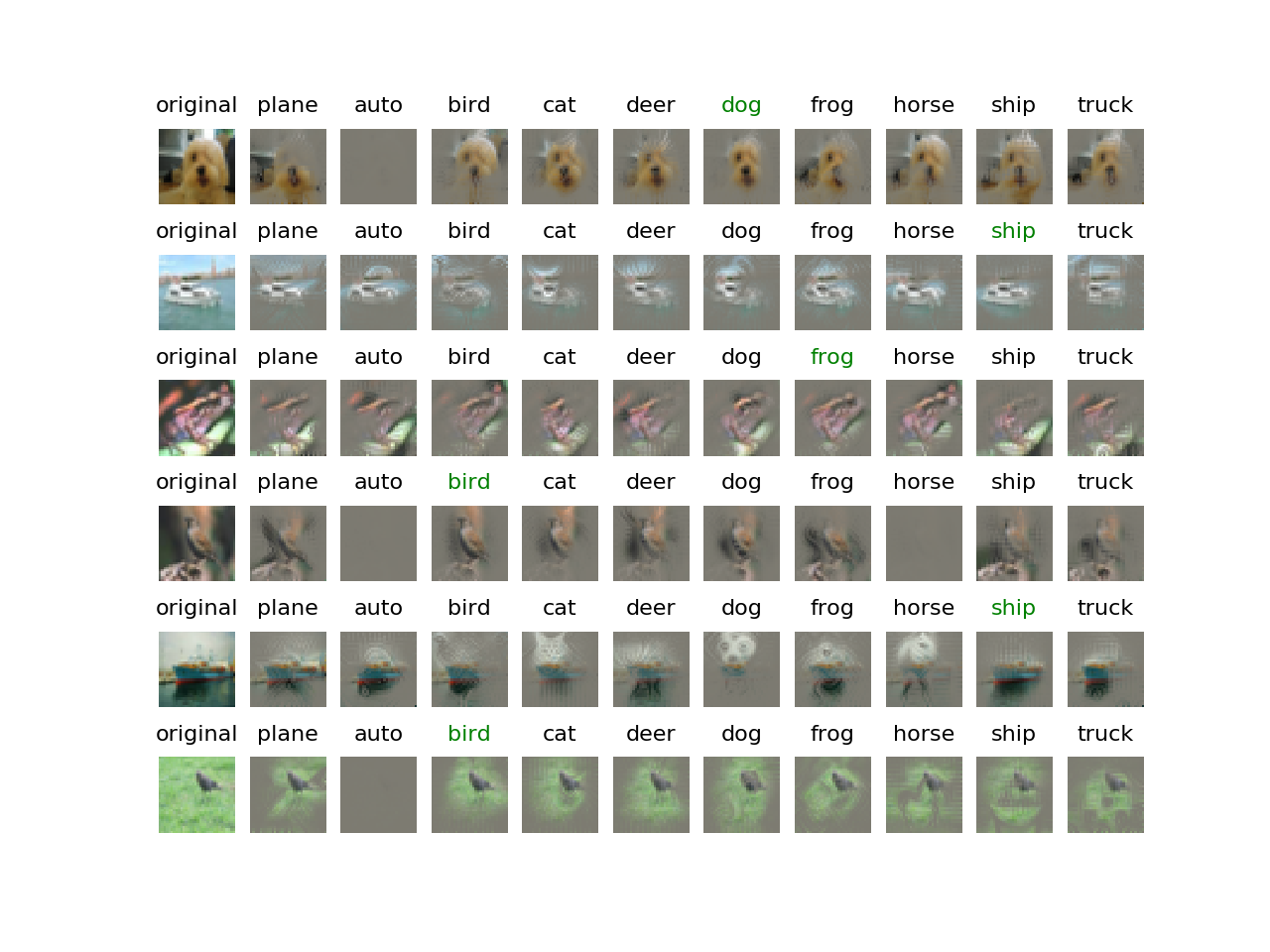}
  \vspace{-0.1in}
  \caption{Our method with TV regularization $\lambda_{TV}= 0.01$ }
  \label{fig:artifact-0.01}
\end{subfigure}

\caption{A demonstration of artifacts created by masking on CIFAR-10. Pixels (partially) masked out are filled with gray based on the fractions they are masked out. Masks generated without or only with low level regularization can easily produce artifacts. It is more common and/or severe for the incorrect label than correct label. }
\label{fig:artifact}
\end{figure}

\subsection{CIFAR-100 Experiments}
\label{subsection:cif100}

\begin{figure}[t!]
\centering
\includegraphics[scale = 0.25]{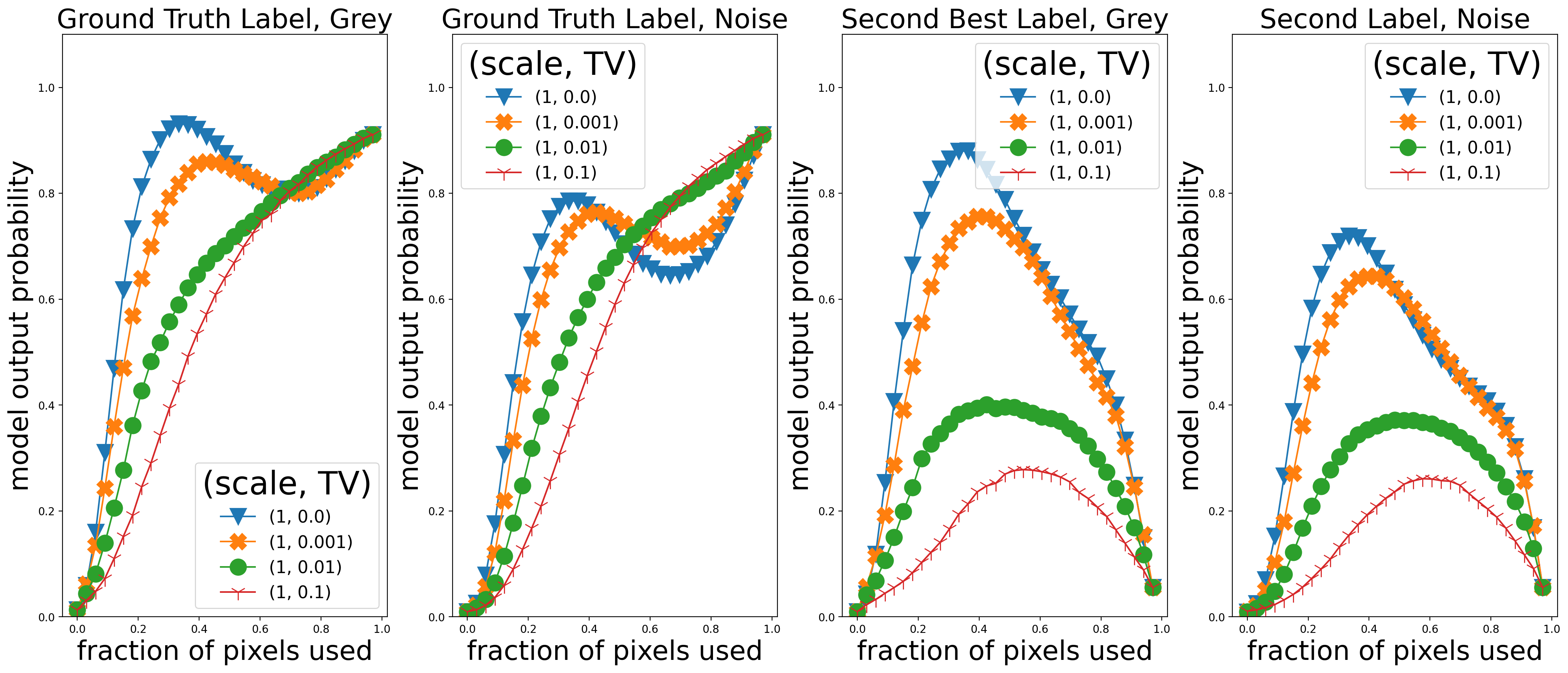}
\caption{[CIFAR-100] AUC curves with as the fraction of pixels retained from the original images based on the mask varies from 0 to 1.0 on the X-axis. The probabilities assigned by the model (averaged over 1600 images) on the Y-axis. {\bf Left: } Mask learned for ground truth label, probabilities for ground truth label while replacing remaining pixels with gray. {\bf Center Left: } Mask learned for ground truth label, probabilities for ground truth label while replacing remaining pixels with other image pixels. {\bf Center Right: } Mask learned for second best label, probabilities for second best label while replacing remaining pixels with gray. {\bf Right: } Mask learned for second best label, probabilities for second best label while replacing remaining pixels with other image pixels. We see that increasing TV regularization results in only a mild drop in completeness, but significantly improves soundness.}
	\label{fig:cif100_auc}	
\end{figure}

\begin{table}[t!]
\begin{center}
	\caption{Completeness and soundness for a ResNet-164 trained model on CIFAR-100, as defined in \Cref{eq:cs}.
	Each column represents mask learned using our procedure, with different TV regularization strengths (0.0, 0.001, 0.01, or 0.1).
	``Gray'' indicates pixels were grayed during AUC calculation and ``Random image'' indicates they were replaced with other images.}
	\label{table:cif100_complete_sound}
	\vspace{0.1in}
\begin{tabular}{ |c|c|c|c| c| } 
 \hline
 Gray & TV = 0.0 & TV= 0.001 & TV= 0.01 & TV= 0.1   \\ 
 \hline
Completeness ($\alpha$)       & 0.80   & 0.74  & 0.64    &  0.55    \\

Soundness ($\beta$)     &  0.28   & 0.34   &  0.58  & 0.75   \\ 
 \hline
 Random image & TV = 0.0 & TV= 0.001 & TV= 0.01 & TV= 0.1   \\ 
 \hline
Completeness ($\alpha$)       &   0.67   & 0.66   &   0.61    &  0.54   \\
Soundness ($\beta$)     &  0.35   & 0.40   & 0.61  & 0.78   \\
 \hline
\end{tabular}
\end{center}
\end{table}

\begin{table}[t!]
\begin{center}
	\caption{Performance of our method on CIFAR-100 and some baselines on various intrinsic saliency metrics proposed in prior work.
	Downarrow (uparrow) means lower (higher) is better.
	We find that while both our masks (learned with and without TV) have very good performance on the insertion metric. The deletion and saliency metrics are uninformative in this case, since all methods are as good (or worse) compared to a random mask.}
	\label{table:cif100}
	\vspace{0.1in}
\begin{tabular}{ |c|c|c|c| c| c| } 
 \hline
 & Gradient $\odot$ Input & Our method & Our Method & Smooth-Grad & Random \\ 
 &  & ($\lambda_{TV}=0.01$) & ($\lambda_{TV}=0$) & saliency &   \\ 
 \hline
 Deletion $\downarrow$       &    0.10   &  0.17   &     0.10     &    0.29    &  0.11   \\

Insertion (blur) $\uparrow$      &   0.36   & 0.71   &   0.82  & 0.39 & 0.20 \\ 
Insertion (gray) $\uparrow$      &   0.27   & 0.62   &    0.76 &   0.29 & 0.11 \\ 
Saliency Metric $\downarrow$        & 0.77    & 0.77     & 0.77   &     0.79  &  0.77 \\ 
 \hline
\end{tabular}
\end{center}
\end{table}

We run the same experiment for CIFAR-100 using the corresponding ResNet\-164 model.
We visualize the masks learned for the correct label in \Cref{subfig:cif100true} and in \Cref{subfig:cif100sbest} we visualize the same for the second best label predicted by the ResNet-164 model.
The AUC curves in \Cref{fig:cif100_auc} suggest a similar trend to that of Imagenette, adding TV regularization results in only a mild drop in completeness, but significantly improves soundness.
Evaluation of our masks, compared to some gradient baselines, on intrinsic metrics can be found in \Cref{table:cif100}. We place a downarrow after the name of the metric to indicate a lower value is considered better and an uparrow when a higher value is considered better. We evaluate on a randomly selected subset of 1600 data points where the model had correct top 1 accuracy.
We report the completeness and soundness results for CIFAR-100 in Table \ref{table:cif100_complete_sound} for TV values in $(0.0,0.001,0.01,0.1)$ calculated using a ResNet-164 model.

\begin{figure}[t!]
\centering
\begin{subfigure}[t]{.85\textwidth}
\centering
\includegraphics[width=\textwidth]{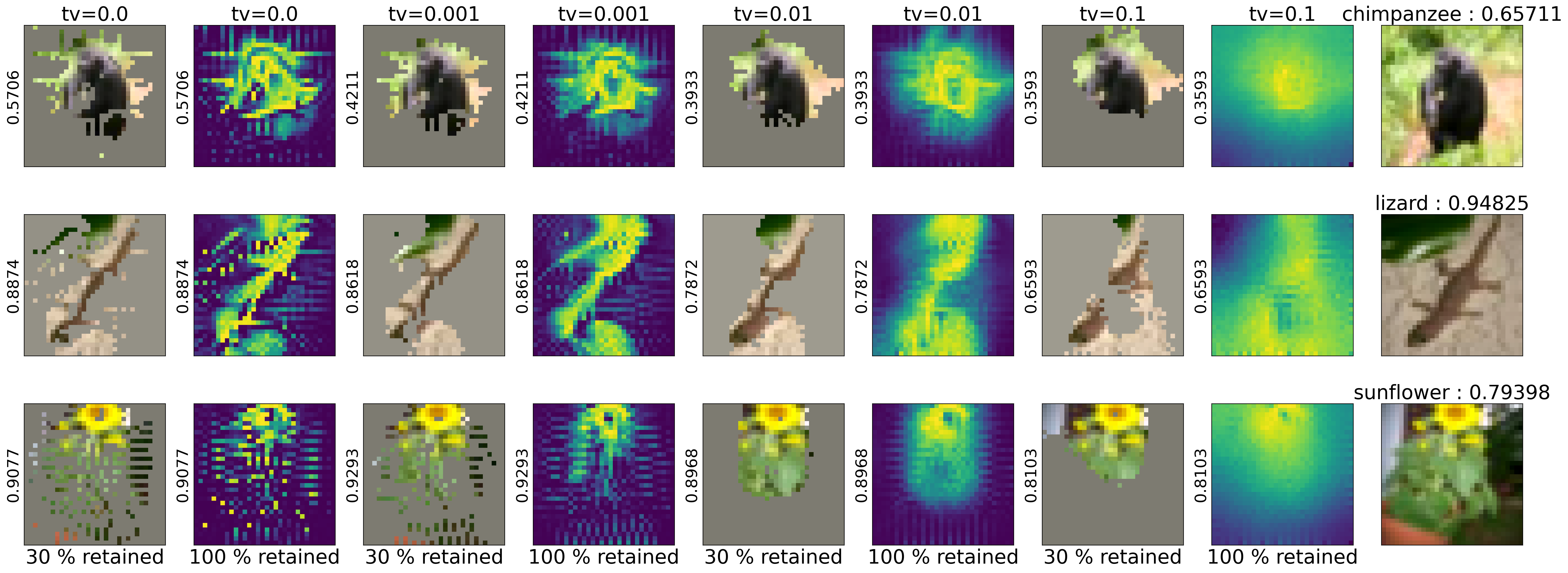}
\caption{}
\label{subfig:cif100true}
\end{subfigure}
\begin{subfigure}[t]{.85\textwidth}
\centering
\includegraphics[width=\textwidth]{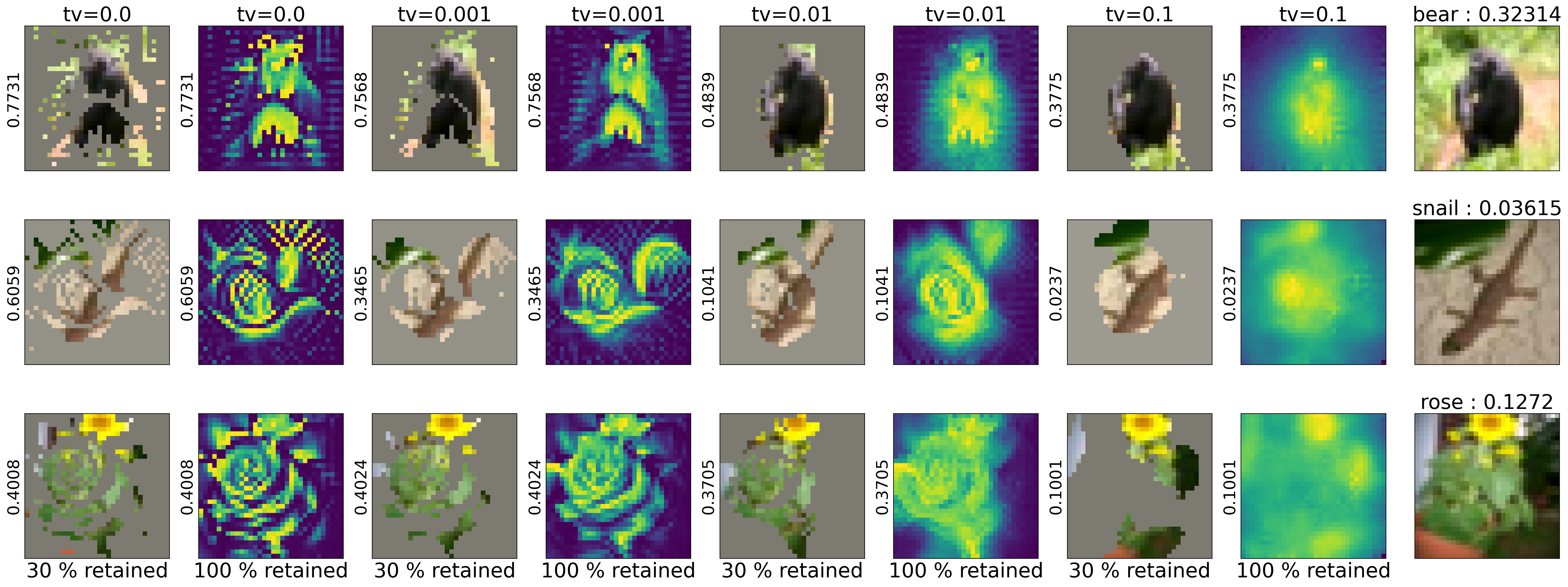}
\caption{}
\label{subfig:cif100sbest}
\end{subfigure}
\caption{Details in \Cref{subsection:cif100} {\bf Panel \ref{subfig:cif100true} } Masks learned for the correct label of CIFAR-100 images using the procedure outlined in \Cref{sec:procedures} on ResNet-164. Columns (1,3,5,7) depict masked images at 30 (retained) \% mask sparseness. Columns (2,4,6,8) depict the original mask. TV values shown above. Original image shown in rightmost column. Model probability of correct label for masked images on y axis. {\bf Panel \ref{subfig:cif100sbest}} Masks learned for the second most probable label of CIFAR-100 images using the procedure outlined in \Cref{sec:procedures} on ResNet-164. Columns (1,3,5,7) depict masked images at 30 \% mask sparseness. Columns (2,4,6,8) depict the original mask.TV values shown above. Original image shown in rightmost column. Model probability of second best label for masked images on y axis.}
	\label{fig:cif100_first}	
\end{figure}

\subsection{Experiments on ImageNet}

\label{subsection:rn18ImageNet}
In \Cref{subfig:imgtrue} we depict the the masks for TV values in $\{0.0,0.01\}$ for a ResNet-18 model on ImageNet for the ground truth label and in \Cref{subfig:imgsbest} we depict the same for the second best label. We also experiment with the effect of upsampling (US) the mask, whereby we learn a mask of size (56,56) and upsample to size (224,224). We use a fixed L1 regularization value of 2e-5.
We depict our results on ImageNet and ResNet-18 in \Cref{table:ImageNetResNet-18} 

For the deletion metric, we note that most methods have comparable or worse performance than the random mask, which suggests that the metric does not give us much signal about the goodness of the saliency maps.
On the insertion metric, we find that mask learned by not adding the TV penalty significantly beats other methods.
The mask learned using TV penalty, on the other hand, has impressive performance on both the insertion AUC and saliency metric (SM).

\myparagraph{Completeness and Soundness on ImageNet and ResNet-18.} We report our results in \Cref{table:ImageNet_complete_sound} for TV values in $(0, 0.01)$ for both graying (Gray) and replacing with other image pixels (Random image). Additionally, we investigate the effect of upsampling (US) where we derive a (56,56) and upsample by a factor of 4 to a $(224,224)$  mask.

\begin{table}[t!]
\begin{center}
	\caption{Completeness and soundness for a ResNet-18 model on ImageNet as defined in \Cref{eq:cs}.
	Each column contains a represents mask learned using our procedure, with our without upscaling and different TV regularization strengths.
	``Gray'' indicates pixels were grayed during AUC calculation and ``Random image'' indicates they were replaced with other images.
	No US indicates the full (224,224) mask was derived and US indicates a $(56,56)$ mask was derived then upsampled by a factor of $4$.
	TV indicates a TV regularization value of 0.0 or 0.01.}
	\label{table:ImageNet_complete_sound}
	\vspace{0.1in}
\begin{tabular}{ |c|c|c|c| c| } 
 \hline
Gray & TV = 0.0 & TV = 0.01 & US TV = 0.0 & US TV = 0.01  \\ 
 \hline
Completeness ($\alpha$)     & 0.97   & 0.76  & 0.87    &  0.71    \\

Soundness ($\beta$)     &  0.19   & 0.70   &  0.38  & 0.75   \\ 

\end{tabular}\\
\begin{tabular}{ |c|c|c|c| c| c| } 
 \hline
Random image & TV = 0.0 & TV = 0.01 & US TV = 0.0 & US TV = 0.01   \\ 
 \hline
Completeness ($\alpha$)    & 0.89   & 0.61  & 0.74    &  0.59    \\

Soundness ($\beta$)     &  0.25   & 0.83   &  0.52  & 0.86   \\

 \hline
\end{tabular}
\end{center}
\end{table}

\begin{table}[t!]
\begin{center}
	\caption{Performance of our method on ImageNet and ResNet-18 model and some baselines on various intrinsic saliency metrics proposed in prior work.
	We find that while both our masks (learned with and without TV) have very good performance on the insertion metric, the mask learned with TV has much better performance on the saliency metric.
	The deletion metric is uninformative in most cases, since most methods are as good (or worse) compared to a random mask.}
	\label{table:ImageNetResNet-18}
	\vspace{0.1in}
\begin{tabular}{ |c|c|c|c| c| c| } 
 \hline
 & Gradient $\odot$ Input & Our method & Our Method & Smooth-Grad & Random \\ 
 &  & ($\lambda_{TV}=0.01$) & ($\lambda_{TV}=0$) & saliency &   \\ 
 \hline
 Deletion $\downarrow$       &    0.10   &  0.13  &     0.21     &    0.08 & 0.13   \\

Insertion (blur) $\uparrow$      &   0.44   & 0.79   &   0.85  & 0.51 & 0.31 \\ 
Insertion (gray) $\uparrow$      &   0.30   & 0.67  &    0.92        & 0.35 & 0.13 \\ 
Saliency Metric $\downarrow$        & 0.31    & 0.15    & 0.32          &     0.32  &  0.32 \\ 
 \hline
\end{tabular}
\end{center}
\end{table}

\myparagraph{Effect of ensembling.} In order to investigate the effect of ensembling we plot maps in \Cref{fig:ImageNet_varyk} as we vary the number of maps that are ensembled over as $K \in \{1,2,4\}$, where we learn multiple masks such that each of them individually validate the label, but are as disjoint as possible.
We do not upsample (using a scale of 1.0) and we use a fixed L1 regularization of 2e-5 and a fixed TV regularization of 0.0.

\begin{figure}[th!]
\centering
\begin{subfigure}[t]{.85\textwidth}
\centering
\includegraphics[width=\textwidth]{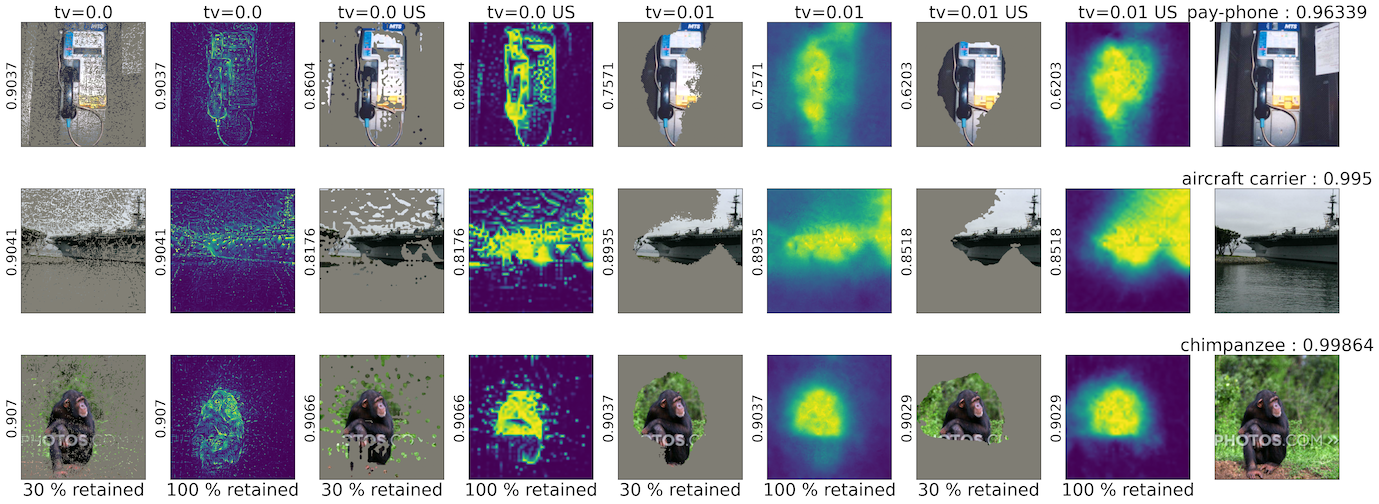}
\caption{}
\label{subfig:imgtrue}
\end{subfigure}
\begin{subfigure}[t]{.85\textwidth}
\centering
\includegraphics[width=\textwidth]{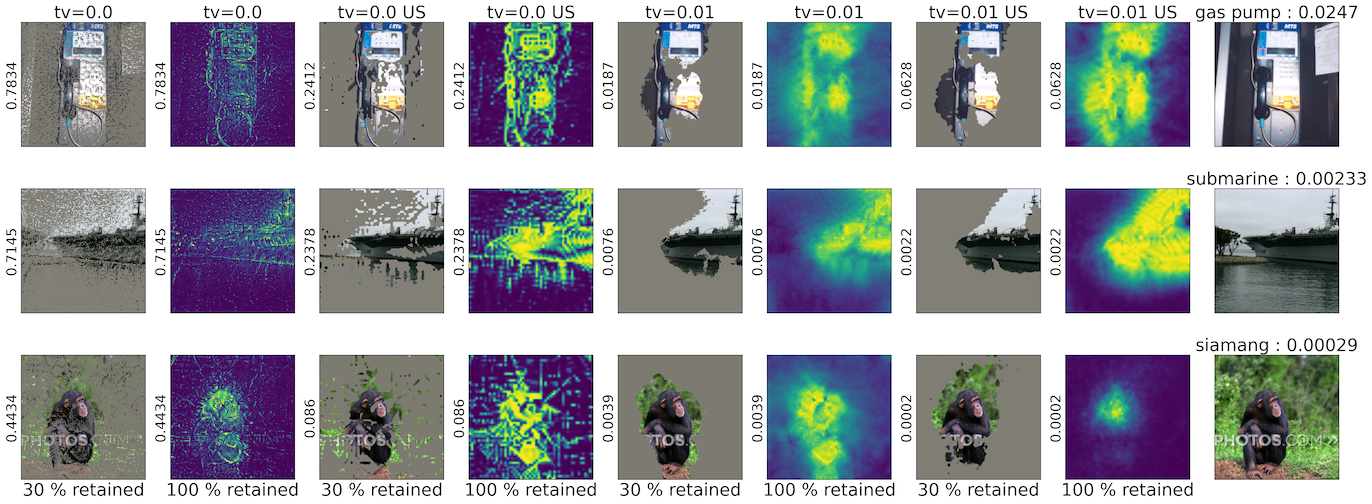}
\caption{}
\label{subfig:imgsbest}
\end{subfigure}
\caption{Details in \Cref{subsection:rn18ImageNet}. US stands for upsampled mask, where we derive a (56,56) mask and interpolate to (224,224). {\bf Panel \ref{subfig:imgtrue} } Masks learned for the correct label of ImageNet images using the procedure outlined in \Cref{sec:procedures} on ResNet-50. Columns (1,3,5,7) depict masked images at 30 (retained) \% mask sparseness. Columns (2,4,6,8) depict the original mask. TV values shown above. Original image shown in rightmost column. Model probability of correct label for masked images on y axis. {\bf Panel \ref{subfig:imgsbest}} Masks learned for the second most probable label of ImageNet images using the procedure outlined in \Cref{sec:procedures} on ResNet-50. Columns (1,3,5,7) depict masked images at 30 \% mask sparseness. Columns (2,4,6,8) depict the original mask.TV values shown above. Original image shown in rightmost column. Model probability of second best label for masked images on y axis. We find, unsurprisingly, that adding TV regularization and upsampling make the mask more continuous and ``human interpretable'' and, more importantly, make it harder to find masks that can get high probability for the second best label, thus ensuring higher soundness.}
	\label{fig:ImageNet_rn18}	
\end{figure}

\begin{figure}[th!]
\centering
\includegraphics[scale = 0.84]{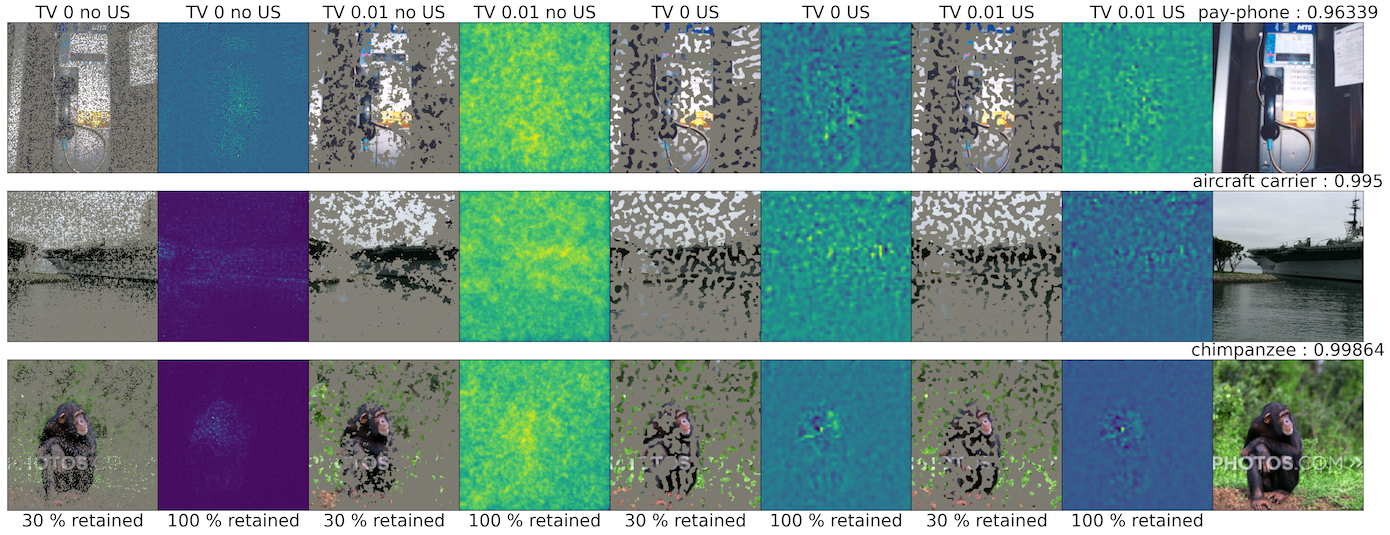}
\caption{Results of randomizing the last layer of a ResNet-18 model on ImageNet data for the procedure described in  \Cref{sec:procedures}. US indicates a $(56,56)$ map was learned and upsampled to $(224,224)$. We find the maps of this randomized network are less visually coherent than the analogous maps of a pre-trained model.}
	\label{fig:ImageNet_sanity}	
\end{figure}

\begin{figure}[th!]
\centering
\includegraphics{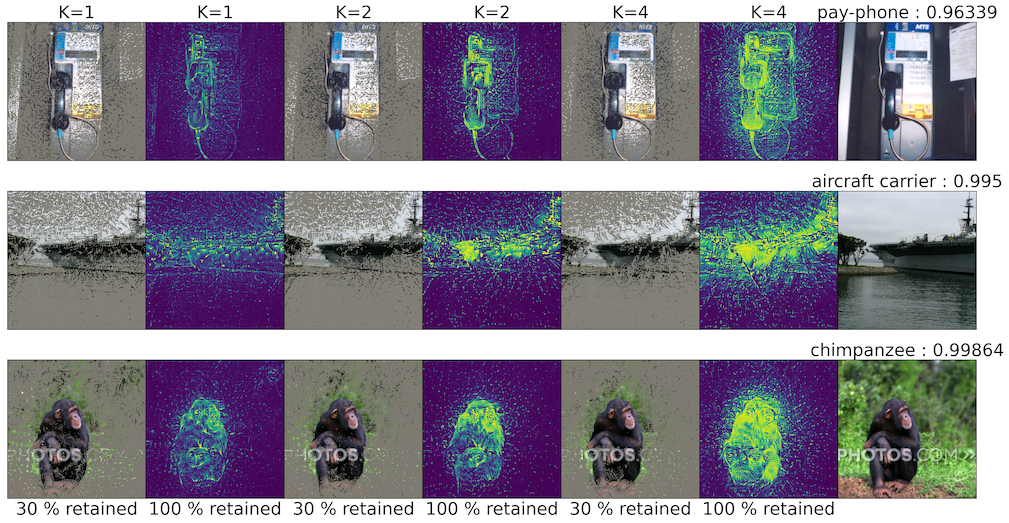}
\caption{{\bf Effect of ensembling} Partial statistical assignments validating the correct label of ImageNet and ResNet-18 images as we vary K, the number of maps. Details in \Cref{subsection:rn18ImageNet}.}
	\label{fig:ImageNet_varyk}	
\end{figure}

\subsection{Effect on Sanity Checks}
\label{subsec:sanity}
Inspired by \citep{adebayo2018sanity} we randomize the last layer of a ResNet-18 network and visually inspect the resulting saliency maps in Figure \ref{fig:ImageNet_sanity}. We find that the maps appear less coherent than those of a pre-trained model. We use a fixed L1 regularization of 2e-5 and depict maps with and without upsampling (US) at TV values of $(0, 0.01)$.

\section{Additional Background information}
\label{sec:additionalbackground}

	\subsection{Additional related work}
	\label{subsec:missing}
	The pixel replacement strategy we used is closely related to hot deck imputation \citep{rubin1976inference}, where features may be replaced either by using the mean feature value (analogous to replacing with grey) or sampling from the marginal feature distribution (analogous to replacing with image pixels sampled from other training images). Some prior work \citep{carter2019made} has found that mean imputation does not significantly affect model output on the beer aroma review dataset. On Imagenette, by contrast, we found that replacement strategy can matter.
	
	\myparagraph{Additional saliency evaluation tests.}
	\citet{adebayo2018sanity}
proposed ``sanity checks'' for saliency maps. They note that if a model's weights are randomized, it has not learned anything, and therefore the saliency map should not look coherent. They also randomize the labels of the dataset and argue that the saliency maps for a model trained on this scrambled data should be different than the saliency maps for the model trained on the original data. We study the effect of model layer randomization on our method in Appendix section \ref{subsec:sanity} and find that randomizing the model weights does cause our saliency maps to look incoherent. 
\citet{tomsett2020sanity} also discover sanity checks for saliency metrics, finding that saliency evaluation methods can yield inconsistent results. They evaluate saliency maps on reliability, i.e. how consistent the saliency maps are. To measure a method's reliability, because access to ground truth saliency maps are not available, they use three proxies 1) inter-rater reliability, i.e. how whether a saliency evaluation metric is able to consistently rank some saliency methods above others, 2) inter-method reliability, which indicates whether a saliency evaluation metric agrees across different saliency methods, and 3)internal consistency reliability, which measure whether different saliency methods are measuring the same underlying concept. \\
\myparagraph{Saliency axioms.} \citet{sundararajan2017axiomatic} identify two fundamental axioms, Sensitivity (versions a and b), and Implementation Invariance that attribution methods should satisfy.  "An attribution method satisfies Sensitivity(a) if for every input and baseline that differ in one feature but have different predictions then the differing feature should be given a non-zero attribution." The definition of Sensitivity b is "If the function implemented by the deep network does not depend (mathematically) on some variable, then the attribution to that variable is always zero.". Implementation invariance simply states that if two networks are equivalent, the saliency maps for those two networks should be the same. These axioms are not captured by completeness and soundness and are good examples of why completeness and soundness cannot be used alone in evaluating saliency maps. Other prior work \citet{carter2019made} argues that saliency explanations should be minimal \citet{carter2019made} and find sufficient input subsets which are "minimal subsets of features whose observed values alone suffice for the same decision to be reached, even if all other input feature values are missing."
 	
	\subsection{Saliency Methods}
We give a partial list of extant saliency methods here. We broadly categorize explanations into three categories: Back-propagation based explanations, axiomatic methods, and masking methods.   { \bf Backpropagation based explanations} shape credit as it is propagated backwards through the neural network according to certain rules.  These approaches include {\bf Layerwise Relevance Propagation} \citep{binder2016layer} which satisfies completeness, {\bf Rect-Grad} which thresholds internal neuron activations \citep{kim2019saliency}, and {\bf DeepLIFT} which satisfies the summation to delta rule.

\myparagraph{Axiomatic methods.} Axiomatic methods decompose the ouput (typically the logit) according to certain axioms like fairness in Shapley based methods {\bf SHAP} \citep{NIPS2017_8a20a862} and {\bf conceptSHAP} \citep{yeh2020completeness}.  We also include gradient based approaches like {\bf Gradient} $\frac{\partial S}{\partial x}$ \citep{baehrens2010explain} which calculates the partial derivative of the logit with respect to the input. {\bf Gradient $\odot$ Input} \citep{shrikumar2017learning} $\frac{\partial S}{\partial x} \cdot x$, which elementwise multiplies the gradient explanation by the input, and {\bf Grad-CAM} \citep{selvaraju2016grad} which takes the gradient of the logit with respect to the feature map of the last convolutional unit of a DNN. {\bf Smooth-Grad} \citep{smilkov2017smoothgrad}, which averages the {\bf Gradient $\odot$ Input} explanation over several noisy copies of the input $x + \eta$, where $\eta$ is some Gaussian.
The previous methods are intrinsic in the sense that they aim to explain the model decision. The last category of saliency maps, namely masking methods, also aim to explain the model decision, but they frequently aim to do so in a way that is interpretable by a human.
Contrastive methods, such as contrastive layerwise propagation \citet{gu2018understanding}, also modify LRP by constructing class specific saliency maps, with the goal of object localization, i.e. in an image of an elephant and zebra, the saliency map for elephant should have high overlap with the elephant, and similarly for the corresponding map for zebra. 

\myparagraph{Masking Methods.} Masking Methods are often evaluated using a pointing game or WSOL metric, which measures overlap with human labeled bounding boxes or explanations. These masking methods include techniques based on averaging over randomly sampled masks \citep{petsiuk2018rise}, optimizing over meaningful  mask perturbations \citep{fong2017interpretable}, and real time image saliency using a masking network \citep{dabkowski2017real}. Pixels that have been removed from the image by the mask may be replaced by graying out, by Gaussian blurring as in \citet{fong2017interpretable}, or with infillers such as CA-GAN \citep{yu2018generative} used in \citet{chang2018explaining}, or DFNet \citep{hong2019deep}. \citet{de2020decisions} find masks using differentiable masking. \citet{taghanaki2019infomask} introduce a method that results in more accurate localization of discriminatory regions via mutual information. 

\myparagraph{Pruning and information theory} \citet{khakzar2019improving} improve attribution via pruning.  \citet{schulz2019restricting} improve attribution by adding noise to intermediate feature maps. 

\myparagraph{Saliency and Boolean Logic.}  Previous work has also drawn connections between saliency and notions in logic. \citet{ignatiev2019relating} relates saliency explanations and adversarial examples by a generalized form of hitting set duality. \citet{ignatiev2019abduction} develops a constraint-agnostic solution for computing explanations for any ML model. \citet{macdonald2019rate} develop a rate distortion explanation for saliency maps and prove a hardness result. \citet{mu2020compositional} find a procedure for explaining neurons by identifying compositional logical concepts. \citet{zhou2018interpreting} describe network dissection, which provides labels for the neurons of the hidden representations. We are unaware of frameworks like Section~\ref{sec:stat}.

\myparagraph{Arguments about saliency.} For discussion including pro/cons  of various methods  some starting points are \citet{seo2018noise}  \citet{fryer2021shapley} \citet{gu2018understanding} \citet{sundararajan2020many}.

\myparagraph{\citet{phang2020investigating}}
We describe separately the masking procedure used by \citet{phang2020investigating}.
They begin by taking a pretrained model on ImageNet. The masker has access to the internal representations of the pre-trained model, and tries to maximize masked in accuracy and masked out entropy. They do not provide the ground truth label to the masker. 

\subsection{Saliency Evaluation Methods}

Saliency evaluation methods attempt to evaluate the quality of a saliency map. Many interpret the heatmap values as a priority order of saliency. 
Extrinsic evaluation metrics include the {\bf WSOL} metric, which aim to measure overlap of the saliency map with a human annotated bounding box and the {\bf Pointing Game} metric proposed by \citet{zhang2018top} in which a pixel count as a hit if it lies within a bounding box and a miss otherwise, and the metric is $\frac{\text{\# Hits}}{\text{\# Hits + \#Misses}}$.
Other more intrinsic methods include early saliency evaluation techniques like {\bf MorF} and {\bf LerF} \citet{samek2016evaluating}, which involve removing pixels either in the order of highest importance or lowest importance and observing the area of the resulting curve. {\bf  Insertion and Deletion Games} of~\citet{petsiuk2018rise} uses this too. The deletion game measures the drop in class probability as important pixels are removed, while the insertion game measures the rise in class probability as important pixels are added. (Our AUC discussion in Section~\ref{sec:stat} relates to this.)
{\bf Remove and Retrain (ROAR)} is a saliency evaluation method proposed by \citet{hooker2019benchmark}. Input features are ranked and then removed according to a saliency map. A new model is trained on the modified training set, and a larger degradation in accuracy on the modified test set compared to the original model on the original test set is regarded as a better saliency method. (NB: retraining makes this a non-intrinsic method.)
Previous work has also introduced datasets specifically designed to test saliency methods. {\bf BAM} \citet{yang2019benchmarking} creates saliency maps by pasting object pixels from MSCOCO \citet{lin2014microsoft} into scene images from MiniPlaces \citet{zhou2017places}.  The  {\bf Saliency Metric} proposed by \citet{dabkowski2017real} thresholds saliency values above some $\alpha$ chosen on a holdout set, finds the smallest bounding box containing these pixels, upsamples and measures the ratio of bounding box area to model accuracy on the cropped image, $s(a,p) = \log(\max(a,0.05)) - \log(p)$ where a is the area of the bounding box and p is the class probability of the upsampled image.  

 	\subsection{Saliency computations and underlying meanings of saliency}
	For simplicity this discussion assumes the datapoints are images and the classifier is a deep net. The heatmap in the saliency method is trying to highlight the contribution of individual pixels to the final answer. This is analogous to how a human may highlight relevant portions of the image with a plan. (Classic saliency methods in vision are inspired by studies of human cognition.) Saliency methods operationalize this intuitive definition in different ways, and we try to roughly categorise these as follows.  
	
	\myparagraph{Variational interpretation.}These interpret saliency in terms of effect on final output due to change in a single pixel --captured either via partial derivative of output with respect to pixel value (i.e., effect of infinitesimal change), or via change of output when this pixel is set to $0$ or to a random (or "gray")  value. Examples include {\bf Gradient}, {\bf Gradient $\odot$ Input}  \citet{shrikumar2017learning}, {\bf Occlusion  }
	
	\myparagraph{Credit attribution guided by gradient.}These  use the gradient to guide the assignment of saliency values. The gradient is interpreted as propagating values from the output to the input layer, and the values are partitioned/recombined at internal nodes of the net following some conservation principles. A key goal is to ensure {\em completeness}, which means that the sum of the attributions equal the logit value.  Examples include {\bf LRP}, {\bf DeepLIFT} \citet{shrikumar2017learning}, {\bf Rect-Grad} Let $a_i^{l}$ be the activation of some node in layer $l$, and $R_i^{l+1}$ be the backpropagated gradient up to $a_i^{l}$. Rect-grad  replaces the vanilla chain rule, $R^{l} = \mathbf{1}[a_i > 0]$ with the rule that $R_{i}^l = \mathbf{1}[R^{l+1}_ia_i > \tau]$ for some threshold $\tau$. Hence, during a backward pass preference is given to nodes with large margin.
	
	\myparagraph{Ensembling on top of above two ideas.} Ensembling methods combine saliency estimates over multiple inputs an an attempt to reduce noise in the final map. Examples include {\bf Smooth-Grad}, Occlusion based methods, etc. We also include {\bf Shapley Values} in this list.
	
	The Shapley value aims to fairly distribute credit among a coalition of $N$ players. In the context of image saliency, each coordinate of the image input may be seen as a player, and the Shapley value computes $\sum_{S \subseteq N} \setminus \{i\} \frac{|S|! (n - |S| -1)!}{n!} (v(S \cup \{i\}) - v(S)) $. It can be interpreted as the marginal contribution of player i, over all possible orderings of the coalition. In this sense, it can be seen as an ensembling method, as it averages over all possible random permutations.

	{\bf Analysis of saliency methods.} Previous work has analyzed ensembling methods like Smooth-grad, and found that it does not smooth the gradient \citet{seo2018noise}. They conclude that Smooth-Grad does not make the gradient of the score function smooth. Rather Smooth-grad is approximately the sum of a standard saliency map and higher order terms and the standard deviation of the Gaussian noise. It has also been found that Shapley values, despite having a uniqueness result, can differ in the way they depend on the model, data, etc \citet{sundararajan2020many}.  \citet{fryer2021shapley} highlight several nuances that should be taken into account when considering Shapley values. They introduce Shapley values as averaging over submodels, and note that "the performance of a feature across all submodels may not be indicative of the particular performance of that feature in the set of optimal submodels.". They provide specific cases where satisfying the axioms of Shapley values works against the goal of feature selection.

\section{Clarifying benefit of TV regularization}
\label{app:linearcase}

This section gives more details of the discussion in \Cref{sec:linear}
about how TV regularizers help ensure soundness even in a linear setting.

Let $\DataS$ be a dataset of labeled data $(\vx, y)$ where the inputs are of unit norm and labels are binary, i.e., $\normtwosm{\vx} = 1$, $y \in \{\pm 1\}$. The model in question is a linear classifier $f(\vx) := \sgn(\dotp{\vw}{\vx})$ parameterized by the weight vector $\vw \in \sphS^{d-1}$, and it achieves the perfect accuracy on the set $\DataS$ with a margin $\gamma := \min_{(\vx, y) \in \DataS} y\dotp{\vw}{\vx} > 0$. We assume that the coordinates of $\vx$ and $\vw$ are uniformly bounded by $\frac{10}{\sqrt{d}}$, i.e., $\norminfsm{\vx} \le \frac{10}{\sqrt{d}}$,  $\norminfsm{\vw} \le \frac{10}{\sqrt{d}}$ ($10$ can be changed to any other constant).


Let $\Gamma$ be the input modification process that sets all non-salient pixels to $0$. We are interested in binary heatmaps, i.e., $\map$ assigns $1$ to pixels in some salient set $S$, and $0$ otherwise. According to
\Cref{def:basemetric}, $g(x, a, \map) = \mathbbm{E}_{\tilde{x} \sim \Gamma(x, \map)} [\onec{f(\tilde{\vx}) = a}]$. A simple calculation shows that this expectation is equal to $\onec{a \sum_{i \in S} w_i x_i > 0}$, and thus the goal is to find $S$ so that $a \sum_{i \in S} w_i x_i > 0$.

As we do not consider the full salient set informative, we are interested in salient sets with size constraint $\abssm{S} = L$ for some $1 \le L \le d$. There is a simple saliency method that achieves this goal: Given an input $\vx$ and a label $a \in \{\pm 1\}$, sort the coordinates according to $a w_i x_i$ and take the highest $L$ coordinates as the salient set $S$.

It is easy to see that this method always produces $S$ with $a \sum_{i \in S} w_i x_i > 0$. Letting $a = y$ proves the completeness. However, this method does not satisfy soundness: a salient set $S$ with $a \sum_{i \in S} w_i x_i > 0$ can also be found for $a \ne y$!

Now we see how the TV constraint helps to ensure soundness (with good probability). A vector can be seen as a 1D image, and the TV of a salient set $S$ can be defined by $\TV(S) := \sum_{i=1}^{d-1} \abs{\onec{i \in S} - \onec{i + 1 \in S}}$.
For simplicity, we consider salient sets with TV at most $2$. This means $S$ is just an interval. Given the size and TV constraints $\abssm{S} = L$, $\TV(S) \le 2$, it is easy to come out with the following saliency method: search over all the intervals of length $L$ and if an interval $S$ satisfies $a \sum_{i \in S} w_i x_i > 0$, return it as the salient set.
Fortunately, this method does satisfy both completeness and soundness, as is justified by \Cref{thm:int1} in \Cref{sec:linear}.

\thmlinear*

\begin{proof}
	Let $\map$ be any fixed interval of length $L$, associated with salient set $S$.
	The distribution of $\sum_{i \in S} w_i x_i$ is identical to the distribution of the sum of $L$ samples drawn from $\{w_1 x_1, \dots, w_d x_d\}$ without replacement. Note that $d yw_1 x_1, \dots, d yw_d x_d$ are $d$ numbers with mean $\gamma$, and their absolute values are bounded by $10^2 = O(1)$. By Chernoff bound,
	\[
	\Pr\left[\frac{1}{L} \sum_{i \in S} d yw_i x_i \le \gamma -\epsilon \right] \le e^{-\Omega(\epsilon^2 L)}.
	\]
	Set $\epsilon = \gamma$ ensures that $y\sum_{i \in S} w_i x_i > 0$ with probability $1 - e^{-\Omega(\gamma^2 L)}$.
	We can fix any interval $\map$ with $L=L_1$ to prove Item 1.
	
	Taking union bounds over all intervals of length $L$, we can see that the probability of existing an interval of length $L$ that certifies $-y$ should be no greater than $\sum_{\abssm{S} = L} e^{-\Omega(\epsilon^2 L)} \le d^2 e^{-\Omega(\gamma^2 L)}$. Setting $L = L_2$ proves Item 2.
\end{proof}

This shows that such salient sets make sense to humans: if the model predicts $y$, then we can find an interval of length $\tilde{\Omega}(1/\gamma^2)$ so that computing the inner product only in that interval leads to the same prediction; otherwise if the model does not predict $y$, such interval cannot be found. Thus it is sufficient to convince humans that the model predicts $y$ by only revealing the existence of such interval and the coordinate values in it.


\end{document}